%% file: main.tex
\documentclass{article}

\usepackage{microtype}
\usepackage{graphicx}
\usepackage{booktabs} 
\usepackage{subcaption}
\usepackage{multirow}

\usepackage{hyperref}

\usepackage[accepted]{icml2024}

\usepackage{amsmath}
\usepackage{amssymb}
\usepackage{mathtools}
\usepackage{amsthm}

\usepackage[capitalize,noabbrev]{cleveref}

\theoremstyle{plain}
\newtheorem{theorem}{Theorem}[section]
\newtheorem{proposition}[theorem]{Proposition}

\theoremstyle{definition}

\theoremstyle{remark}

\usepackage[textsize=tiny]{todonotes}

\usepackage{wrapfig}
\usepackage{array}

\newcommand{\ours}{\texttt{LaLiGAN}}

\usepackage{etoolbox}
\newtoggle{comment}
\togglefalse{comment}

\usepackage{commands}

\icmltitlerunning{Latent Space Symmetry Discovery}

\begin{document}

\twocolumn[
\icmltitle{Latent Space Symmetry Discovery}




\begin{icmlauthorlist}
\icmlauthor{Jianke Yang}{ucsd}
\icmlauthor{Nima Dehmamy}{ibm}
\icmlauthor{Robin Walters}{ne}
\icmlauthor{Rose Yu}{ucsd}
\end{icmlauthorlist}


\icmlaffiliation{ucsd}{UCSD}
\icmlaffiliation{ibm}{IBM Research}
\icmlaffiliation{ne}{Northeastern University}

\icmlcorrespondingauthor{Rose Yu}{roseyu@ucsd.edu}

\icmlkeywords{Machine Learning, ICML}

\vskip 0.3in
]



\printAffiliationsAndNotice{}  

\begin{abstract}
Equivariant neural networks require explicit knowledge of the symmetry group. Automatic symmetry discovery methods aim to relax this constraint and learn invariance and equivariance from data. However, existing symmetry discovery methods are limited to simple linear symmetries and cannot handle the complexity of real-world data. 
We propose a novel generative model, Latent LieGAN (\ours{}), which can discover symmetries of nonlinear group actions.  It learns a mapping from the data space to a latent space where the symmetries become linear and simultaneously discovers symmetries in the latent space. Theoretically, we show that our model can express nonlinear symmetries under some conditions about the group action. Experimentally, 
we demonstrate that our method can accurately discover the intrinsic symmetry in high-dimensional dynamical systems. \ours{} also results in a well-structured latent space that is useful for downstream tasks including equation discovery and long-term forecasting. We make our code available at \href{https://github.com/jiankeyang/LaLiGAN}{https://github.com/jiankeyang/LaLiGAN}.
\end{abstract}

\input{sec/intro}
\input{sec/related_works}
\input{sec/method}
\input{sec/exp1}  
\input{sec/exp3}  
\input{sec/conclusion}

\section*{Impact Statement}
This paper presents work whose goal is to advance the field of Machine Learning. There are many potential societal consequences of our work, none of which we feel must be specifically highlighted here.

\section*{Acknowledgement}
This work was supported in part by  U. S. Army Research Office
under Army-ECASE award W911NF-07-R-0003-03, the U.S. Department Of Energy, Office of Science, IARPA HAYSTAC Program, and NSF Grants \#2205093, \#2146343, \#2134274, \#2107256,  \#2134178, CDC-RFA-FT-23-0069 and DARPA AIE FoundSci.



\bibliography{ref}
\bibliographystyle{icml2024}

\newpage
\appendix
\onecolumn

\input{appendix/more_exp}
\input{appendix/prelim}
\input{appendix/nga-theory}
\input{appendix/exp_detail}

\end{document}

%% file: sec/intro.tex
\section{Introduction}

\begin{figure}
\vspace{-1mm}
\centering
    \includegraphics[width=.75\linewidth, viewport=400 10 1510 1010, clip]{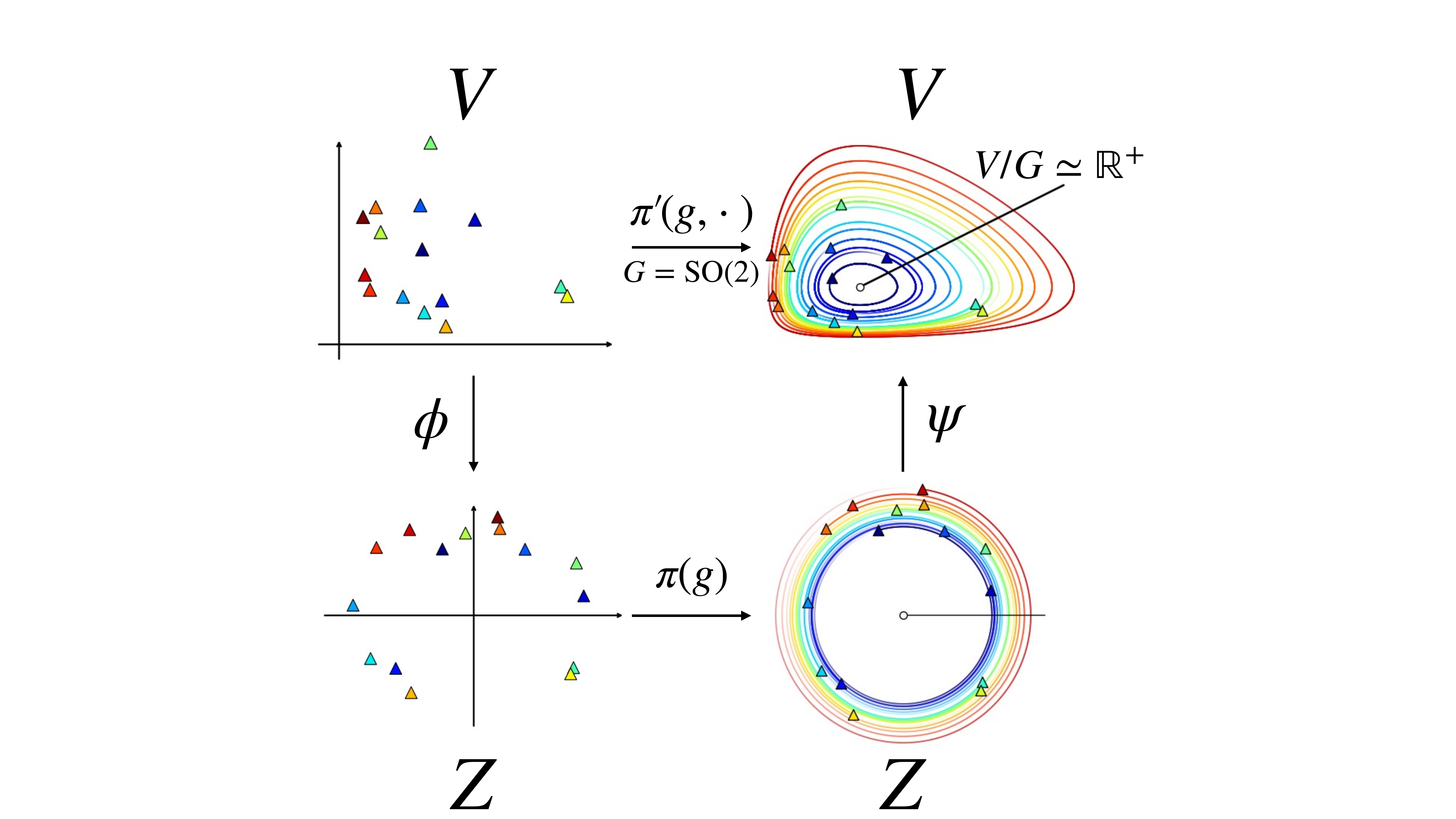}
    \caption{An example of $\mathrm{SO}(2)$ nonlinear group action $\pi'$ on $V=\mathbb R^2$ and its decomposition into an encoder $\phi$, a linear representation $\pi$ and a decoder $\psi$. Each trajectory is a group action orbit containing a random $v \in V$.}
    \label{fig:nga}
    \vspace{-4mm}
\end{figure}

Symmetry plays an important role in the success of deep neural networks \citep{bronstein2021geometric}. Many equivariant networks have been developed to enforce various symmetries in data from images to graphs \citep{e2cnn,gaugecnn,deepsets,lieconv,kondor2018generalization,cohen2019general,emlp,bspline-cnn}. However, a critical limitation of 
existing equivariant networks is that they require knowing the symmetry a priori. For complex real-world data, the underlying symmetries may be unknown or challenging to articulate through programming. For example, dynamical systems can evolve on a low-dimensional manifold with simple symmetries, but the actions of the symmetries become highly nonlinear on high-dimensional observations. Similarly, the action of $\mathrm{SO}(3)$ rotation become complicated on 2D images of 3D objects \cite{garrido2023sie}.

Recent years have seen exciting attempts towards automatic symmetry discovery from data \citep{lconv,moskalev2022liegg,msr,liegan},
but most of them search in only a limited space of symmetries, such as linear actions of discrete and continuous groups. Symmetry discovery is successful only when observations are measured in an ideal coordinate system with linear symmetry. Unfortunately, real-world data is complex and often contain nonlinear symmetries, such as high-dimensional dynamical systems \citep{sindyae}, or 2D images of 3D objects \citep{garrido2023sie}.

Another line of study focuses on learning equivariant representations from data \citep{SEN, yu2022rotationally, E-SSL, quessard2020learning}. These approaches learn a latent embedding space with a given symmetry. However, they still require prior knowledge about the symmetry in the latent space.
They also assume additional information about group transformation associated with each data point, which is not always available in practice.

In this work, we propose a novel generative modeling framework, \ours{}, for discovering symmetries of \textit{nonlinear} group actions. Our key insight is that  nonlinear group transformations can be decomposed into nonlinear mappings between data space and latent space, and a linear group representation in the latent space. Figure \ref{fig:nga} provides such an example. A nonlinear action of $\mathrm{SO}(2)$ on $V=\mathbb R^2$ corresponds to standard 2D rotation on latent vectors $z = \phi(v)$. After decomposition, we can adapt an existing symmetry discovery algorithm such as LieGAN \cite{liegan} to discover linear  symmetries in the latent space.
In the entire process, \ours{} learns both the symmetry group and its action on data. Additionally, when the symmetry group is already known, \ours{} can also be applied to learn the group equivariant representations, with the advantage of not requiring the knowledge of group elements associated with data samples.


The significance of latent space symmetry discovery is multi-fold. From the perspective of symmetry discovery, it further expands the search space of symmetries beyond linear group actions. For representation learning, learning a latent space in which symmetry becomes linear places a strong inductive bias on the structure of latent representations.
Such a simple latent structure proves to be useful in various downstream tasks, such as equation discovery and long-term forecasting in temporal systems. Furthermore, compared to equivariant representation learning, as the symmetry is no longer fixed but learnable, our method can discover latent spaces with previously unknown symmetries.


In summary, our main contributions include:
\begin{itemize}
    \item We develop \ours{}, a novel framework for discovering symmetries of nonlinear group actions.
    \item We provide the theoretical guarantee that \ours{} can approximate any nonlinear symmetry under some conditions about the group action.
    \item Our method can lead to well-structured latent spaces with interpretable symmetries in high-dimensional and nonlinear dynamical systems.
    \item The discovered symmetry can be used for equation discovery, leading to simpler equation forms and improved long-term prediction accuracy.
\end{itemize}

%% file: sec/related_works.tex
\section{Related Works}
\paragraph{Automatic symmetry discovery.} Automatic symmetry discovery aims to search and identify unknown symmetries in data. Current symmetry discovery techniques vary a lot in their search space for symmetries, such as learning discrete finite groups \citep{msr,karjol2023unified}, learning group subsets that represent the extent of symmetry within known groups \citep{augerino,p-gcnn,chatzipantazis2021trm}, and learning individual symmetry transformations on dataset distribution \citep{sgan}. Attempts have been made to discover general continuous symmetries based on Lie theory. For example, L-conv \citep{lconv} works with Lie algebra to approximate any group equivariant functions. LieGG \citep{moskalev2022liegg} extracts symmetry from a learned network from its polarization matrix. LieGAN \citep{liegan} proposes a general framework for discovering the symmetries of continuous Lie groups and discrete subgroups. These methods address general linear group symmetry in the data, which is the largest search space so far. Our work further expands the search space to non-linear symmetries.

\vspace{-1mm}
\paragraph{Learning equivariant representation.} Instead of working in the data space where symmetry transformations can be complicated, many works use autoencoders to learn a latent space with pre-specified symmetries \citep{hinton2011transforming, falorsi2018hvae}.
Among recent works, \citet{yu2022rotationally, SEN} learn equivariant features that can be used for downstream prediction tasks.
\citet{shakerinava2022structuring, E-SSL} use contrastive losses to learn equivariant representations in a self-supervised manner.
\citet{caselles2019symmetry, quessard2020learning, marchetti2022equivariant} focus on learning disentangled representations that are highly interpretable.
\citet{winter2022unsupervised, Wieser2020stib} split the latent space into group-invariant and equivariant subspaces.
While the emphases of these works vary, the common assumption is that \textit{we already know the symmetry group a priori}. Many works also assume additional information such as group element associated with each data point \citep{garrido2023sie} or paired samples under transformations \citep{shakerinava2022structuring}. Our goal is more ambitious: design a model to simultaneously learn symmetries and the corresponding equivariant representations in latent space with minimal supervision.

\vspace{-1mm}
\paragraph{Discovering governing equations.} Latent space discovery of governing equations is first introduced in SINDy Autoencoder \citep{sindyae}, which combines the sparse regression technique for equation discovery in \cite{sindy} and an autoencoder network to explore coordinate transformations that lead to parsimonious equations. Several variants of this method have been developed to improve accuracy and robustness to noise \citep{sindy-pi, weak-sindy, ensemble-sindy}. However, due to the absence of physical constraints, their discovered equations may not respect some physical properties such as isotropy and energy conservation.
We highlight this field as an important application of our symmetry discovery method, where enforcing symmetry can regularize the latent space and improve the performance of equation discovery models.

%% file: sec/method.tex
\section{Representation vs Nonlinear Group Action}
Equivariant neural nets build on the notion of symmetry groups and their transformations on data. Given a vector space $V$, a group $G$ transforms $v \in V$ via a group action $\pi: G \times V \rightarrow V$ which maps the identity element $e$ to identity transformation, i.e. $\pi(e, v) = v$, and is compatible with group composition, i.e. $\pi(g_1, \pi(g_2, v)) = \pi(g_1g_2, v)$.

Many existing equivariant networks assume that the group acts linearly on the input vector space. Examples include $\mathrm{E}(2)$ symmetry acting on planar image signals \cite{e2cnn}, and $\mathrm{SO}(3)$ symmetry acting on spherical signals \cite{spherecnn}. In these cases, the linear group action is called a group representation. The group representation is defined as a map $\rho: G \rightarrow \mathrm{GL}(n)$ where $\rho(g) \in \mathbb R^{n\times n}$ is an invertible matrix that transforms any vector $v \in \R^n$ by matrix multiplication. Given the group representations on the input and the output spaces, a $G$-equivariant network $f:X \rightarrow Y$ needs to satisfy $\rho_Y (g) f(x) = f(\rho_X(g) x)$. A special case of equivariance is invariance, where the group action on the output space is trivial, i.e. $\rho_Y(g) = \mathrm{id}$.

Equivariant networks with such linear symmetry transformations have several limitations. It is not always possible to find a linear action of the group on the data, e.g. the action of $\mathrm{SO}(3)$ on 2D images of 3D objects. Also, we may not even know the symmetry group $G$, so learning equivariant representations for known groups is also not an option.

Our goal is to discover both the \textbf{symmetry group} and its \textbf{nonlinear group action} on the data. Concretely, given the input and output data space $X \subseteq \mathbb R^n,\ Y \subseteq \mathbb R^m$, and the data samples $(x_i,y_i) \in X \times Y$ with an underlying function $y=f(x)$, we want to find a group $G$ and its nonlinear actions $\pi'_X: G \times X \rightarrow X$ and $\pi'_Y: G \times Y \rightarrow Y$ such that $\pi'_Y (g, f(x)) = f(\pi'_X(g, x))$. We denote nonlinear group actions as $\pi'_\cdot$ to distinguish them from group representations.
In the following sections, we will also refer to group representations and nonlinear group actions as linear symmetries and nonlinear symmetries. 

We use the theory of Lie groups to describe the continuous symmetry groups of data. We provide some preliminaries about Lie groups and their representations in Appendix \ref{sec:lie-prelim}.

\section{\texttt{LaLiGAN}: Discovering Nonlinear Symmetry Transformations}

\subsection{Decomposing the Nonlinear Group Action} \label{sec:decomp}

Our major goal is to model a nonlinear action of a group $G$ on a data manifold $\mc{M}$: $\pi': G \times \mc{M} \rightarrow \mc{M}$. We adopt the manifold hypothesis \cite{bengio2013representation} which states that high-dimensional data dwell in the vicinity of a low-dimensional manifold embedded in the high-dimensional vector space, i.e. $\mc{M} \subseteq V = \mathbb R^n$. If we use a neural network $f_\theta$ to directly approximate this function, it cannot guarantee the identity and compatibility conditions for proper group action, i.e. $f_\theta(\mathrm{id}, x) = x$ and $f_\theta (g_1, f_\theta (g_2, x))=f_\theta (g_1g_2, x)$. Instead, we propose to decompose the nonlinear group action as nonlinear maps and a linear group representation. Concretely, we represent any nonlinear group action $\pi': G \times \mc{M} \rightarrow \mc{M}$ as
\begin{equation}
\pi'(g, \cdot) = \psi \circ \pi(g) \circ \phi,
\label{eq:decomp}
\end{equation}
where $\phi: V \rightarrow Z$ and $\psi: Z \rightarrow V$ are functions parametrized by neural networks, and $\pi(g): G \rightarrow \mathrm{GL}(k)$ is a group representation acting on the latent vector space $Z=\mathbb R^k$, where $k$ is a hyperparameter. 

Intuitively, the decomposition \eqref{eq:decomp} projects the data to a latent space where the symmetry group acts linearly, and lifts the transformed latent vector back to the input space.
We provide theoretical guarantees for the expressivity of such a decomposition. Theorem \ref{thm:uanga} indicates that our proposed decomposition and neural network parametrization can approximate nonlinear group actions under certain conditions.

\begin{figure*}[t]
    \centering
    
    \begin{subfigure}{.8\textwidth}
    \includegraphics[width=\textwidth, viewport=100 280 1800 1000, clip]{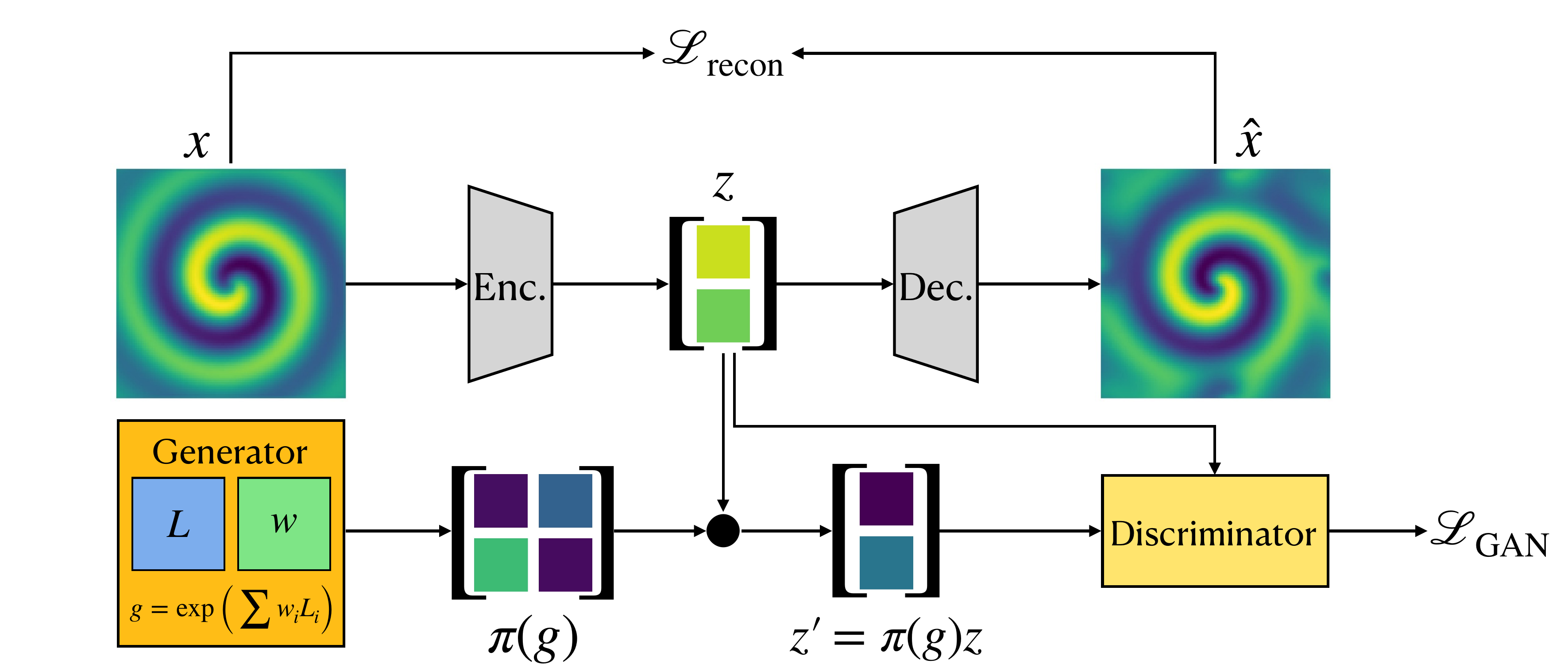}
    \end{subfigure}
    \caption{Overview of the proposed \ours{} framework. The encoder maps the original observations to a latent space. The latent representation is transformed with the linear group action from the generator. The decoder reconstructs the inputs from original and transformed representations. The discriminator is trained to recognize the difference between the original and the transformed samples.}
    \label{fig:structure}
    \vspace{-2mm}
\end{figure*}

\begin{theorem}[Universal Approximation of Nonlinear Group Action]\label{thm:uanga}
    Let $G\leq\mathrm{GL}(k; \mathbb R)$ be a compact Lie group that acts smoothly, freely and properly via a continuous group action $\pi':G\times \mc{M}\rightarrow \mc{M}$, where the data manifold $\mc{M}$ is a compact subset of $V=\mathbb R^n$. The group action, restricted to any bounded subset of the group, can be approximated by the decomposition $\pi'(g,\cdot)\approx\psi\circ\pi(g)\circ\phi$ if it admits a simply connected orbit space $\mc{M}/G$, where $\psi$ and $\phi$ are fixed arbitrary-width neural networks with one hidden layer, and $\pi$ is a linear group representation.
\end{theorem}

\vspace{-1mm}
\textbf{Proof Sketch.} We construct a mapping $\mc M \to \mc M / G \times G$ for any $v \in \mc M$. Based on this mapping, we define a continuous function $\alpha$ from the data manifold to the latent vector space. $\alpha$ can be continuously extended to the ambient space $V = \R ^n$, so it can be approximated by a neural network according to the Universal Approximation Theorem. An inverse mapping $\beta: Z \to \mc M$ and its neural network approximation $\psi$ can be constructed similarly. Full proof is deferred to \cref{sec:uanga}.

\vspace{-1mm}
\subsection{Training Objective for Latent Symmetry Discovery}

Note that \eqref{eq:decomp} alone is not a valid definition of the group action on $\mc M$. In this section, we propose our model architecture and training objective to learn proper symmetries. The proof of \cref{thm:uanga} via constructing a pair of inverse mappings provides insights into how to make \eqref{eq:decomp} satisfy the group action axioms. Concretely,
\begin{proposition}\label{prop:nga}
    $\pi'(g,\cdot) = \psi \circ \pi(g) \circ \phi \big|_\mc{M}$ is a group action on $\mc{M}$ if (1) $\phi \big|_\mc{M}$ is the right-inverse of $\psi$, and (2) the image of $\mc M$ under $\phi$ is invariant under the action $\pi$ of $G$, i.e. $G\phi[\mc M] = \phi[\mc M]$.
\end{proposition}

The condition that $\phi$, when restricted to the data manifold, is the right-inverse of $\psi$ implies that they form an autoencoder that maps between the input vector space and the latent space. In practice, we train the networks $\phi$ and $\psi$ with a reconstruction loss $\mc L_\text{recon}=\mathbb{E}_{v\sim p_\mc{M}}\|\psi(\phi(v))-v\|^2$ to enforce this condition. With the manifold hypothesis, even when the ambient space $V$ has higher dimensionality than the latent space, it is still possible to find the bijective mappings between the data points and the latent embeddings.

To enforce the second condition, i.e. the invariance of the data manifold projection onto the latent space under group action, we apply the approach of LieGAN \cite{liegan} to our latent space. Concretely, we use a symmetry generator to generate linear transformations $\pi(g)$ on the latent vectors. The discriminator is trained to distinguish the original data distribution and the transformed distribution in the latent space. Through adversarial training, the generator learns to produce group actions that preserve the data distribution, i.e. $p_{\mc M}(\phi(v)) \approx p_{\mc M}(\pi(g)\phi(v))$. If the supports of the two distributions agree, the second condition $G\phi[\mc M] = \phi[\mc M]$ is fulfilled. Thus, we use the following training objective to discover symmetry from the data:
\begin{eqnarray}
    &\mc L_\text{total}=w_\text{GAN}\cdot \mc L_\text{GAN}+w_\text{recon}\cdot \mc L_\text{recon}, \\ &\mc L_\text{recon}=\mathbb{E}_{v}\|(\psi(\phi(v))-v\|^2,  \\
    &\mc L_\text{GAN}=\mathbb{E}_{v,g}\Big[\log D(\phi(v))+ \log(1- D(\pi(g)\phi(v))\Big] \label{eq:gan-loss}
\end{eqnarray}
where $D$ is the discriminator, $\pi(g)$ is a linear representation sampled from the generator, and $\phi$ and $\psi$ are neural networks that compose the nonlinear group action with $\pi(g)$. The discriminator, the generator and the autoencoder are jointly optimized under $\mc L_\text{total}$. The loss weighting coefficients $w_\text{GAN}$ and $w_\text{recon}$ are selected based on specific tasks. Figure \ref{fig:structure} shows the overall pipeline of our framework.

To discover equivariance from data, we concatenate the input-output pair of the function as $v=(x,y)$ and let the group act on the concatenated vector by $\pi'(g, v) \coloneqq (\pi'_X(g, x), \pi'_Y(g, y))$, where $\pi'_X(g, \cdot) = \psi \circ \pi_X(g) \circ \phi_X$ is the nonlinear action on the function input space $X$ and $\pi'_Y$ the action on the output space $Y$.
In some tasks such as the dynamical systems considered in \cref{sec:exp-dynamics}, we assume the group action is the same on $X=Y=\R^n$. In this case, we only need to learn a single group action for both $X$ and $Y$.

We should also note that while the above objective encourages our model to conform to the conditions in \cref{prop:nga}, it is difficult to strictly satisfy these properties. In practice, even when these conditions do not hold perfectly, we can still learn a mostly valid group action with reasonably small violations to identity and compatibility axioms. We show an example in \cref{sec:g-action-approx}.

\vspace{-1mm}
\subsection{The Symmetry Generator}
The discriminator and the autoencoder can be instantiated as standard neural architectures, such as MLP between two vector spaces. Here, we discuss how to instantiate the symmetry generator. We use the generator to model a group $G \leq\mathrm{GL}(k)$ which acts on the latent space $\R^k$ via its standard representation $G \to \R^{k\times k}$. Similar to \citet{liegan}, our generator learns a Lie algebra basis $\{ L_i \in \R^{k \times k}\}$ and generates the standard representations of group elements by sampling the linear combination coefficients $w_i\in\R$ for the basis:
\begin{equation}
    w_i\sim\gamma(w),\ \pi(g)=\exp\big[\sum_iw_iL_i\big]
    \label{eq:sym-gen}
\end{equation}
where $\gamma$ is a distribution (e.g. Gaussian) for the coefficients and $\exp$ denotes the matrix exponential. As the Lie algebra basis $\{L_i\}$ uniquely determines the structure of the Lie group, we can learn the symmetry group by learning these $L_i$ via standard gradient-based optimization techniques.

Then, the symmetry generator \eqref{eq:sym-gen} samples random group elements that transform the latent projections of the data points $v_i=(x_i,y_i)$. We use the original and the generator-transformed latent embeddings as the ``real'' and ``fake'' samples for the discriminator. As the generator produces linear representations acting on the latent space, we name our method Latent LieGAN (\ours{}).

\subsection{Structuring the Latent Space}\label{sec:latent-structure}
The latent space produced by the encoder can be largely arbitrary, even with the GAN loss \eqref{eq:gan-loss} that promotes symmetry. Here, we introduce several techniques to endow the latent space with desirable structures for symmetry discovery.

\paragraph{Disentangled representation.}
Latent space representations may capture different aspects of the observations. Consider an image of $N$ 3D objects as an example. A possible latent representation consists of the orientation of each object $r_o \in \mathbb R^{3N}$, the camera perspective $r_c \in \mathbb R^3$, light intensity $i \in \mathbb R^+$, etc. Each component can be transformed by a separate group action, independent of each other. For these scenarios, we provide the option to specify how the latent space is decomposed as independent subspaces, i.e. $Z=\oplus_{i=1}^N Z_i$, each of which is acted on by a symmetry group $G_i$. This avoids searching in the unnecessarily large space of group actions with no nontrivial invariant subspace.
This aligns with the notion of disentangled representation in \citet{higgins2018disentangled}. We discuss the relation and difference between our method and symmetry-based disentangled representation learning in Appendix \ref{sec:flatland}.

\begin{figure}[h]
  \centering
  \begin{subfigure}{0.27\textwidth}
  \includegraphics[width=\linewidth, viewport=30 -30 550 320, clip]{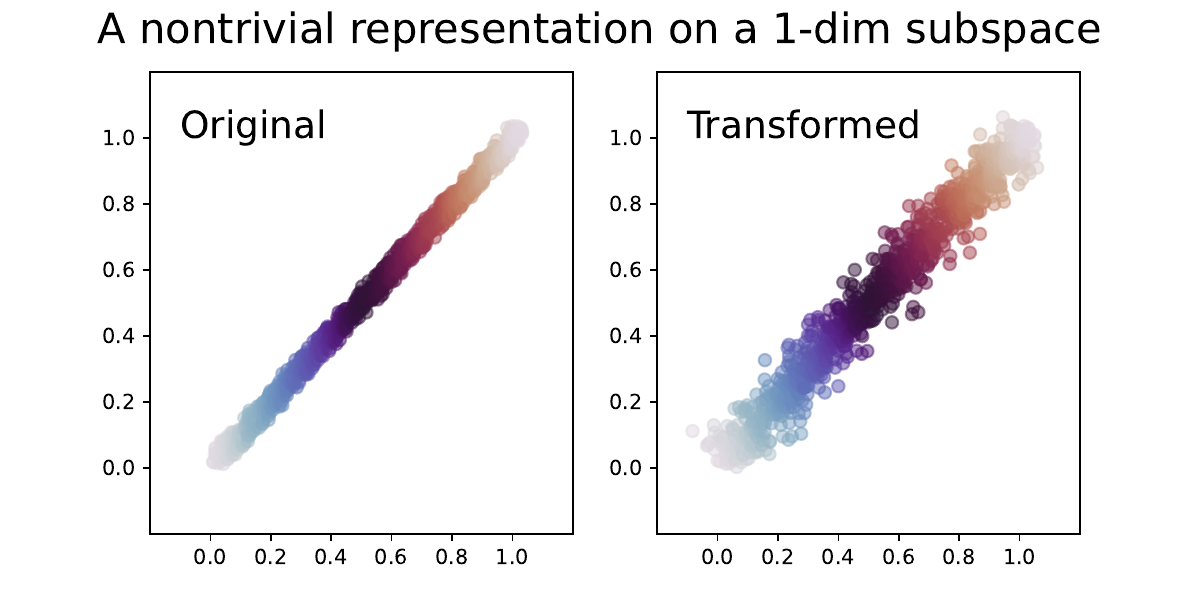}
  \end{subfigure}
  \begin{subfigure}{0.185\textwidth}
  \includegraphics[width=\linewidth, viewport=30 22 450 360, clip]{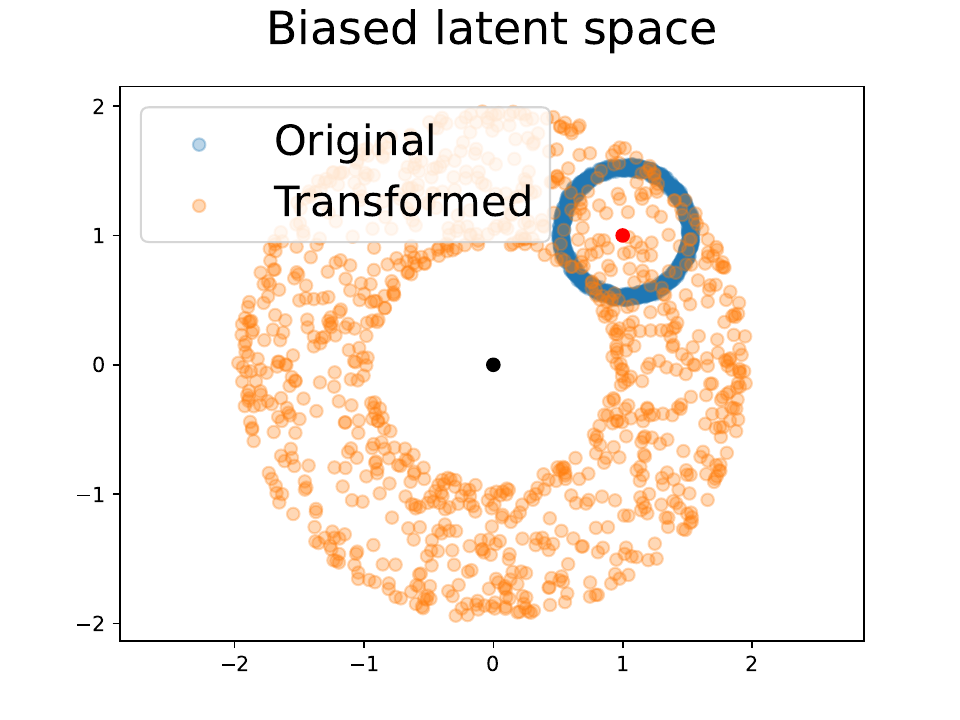}
  \label{fig:fail-bias}
  \end{subfigure}
  \caption{Potential failure modes in latent space symmetry discovery. (a) Fallacious symmetry in low-dimensional subspace. (b) Absence of symmetry in a biased latent space.}
  \label{fig:failure}
\end{figure}

\paragraph{Regularizing the latent structure.}\label{sec:latent-reg}
We observe that the learned latent space can sometimes lead to fallacious symmetry or no symmetry at all. We propose regularization techniques to address a few failure modes caused by undesirable latent space structures. 

First, the latent representations tend to collapse to a low-dimensional subspace where nontrivially parametrized group representations can act as identity. Such a fallacious symmetry provides an easy workaround for the symmetry generator. This happens in \cref{fig:failure} (left), where the transformations generated by \(L = \begin{bmatrix}
    2 & -2 \\ -1 & 1
\end{bmatrix}\) leave the latent representations in a 1D subspace approximately unchanged. This is undesirable because we want the symmetry generator to learn nontrivial transformations. In practice, we use an orthogonal parametrization in the final linear layer of the encoder to enforce a different output in each dimension. This is implemented in \texttt{PyTorch} by computing a product of Householder reflectors to obtain orthonormal rows in the weight matrix.

Another failure mode occurs when the latent representations are not centered at the origin. The linear group representation $v\mapsto \pi(g)v$ implicitly assumes that the vector space is centered at the origin and cannot describe the symmetry otherwise. \cref{fig:failure} (right) provides an example of a circular latent space centered at $(1,1)$. Directly applying the $\mathrm{SO}(2)$ transformations changes the distribution. Thus, we enforce the center property by normalizing each batch of data to have zero means before applying the transformations from the symmetry generator.

\subsection{Use Cases of Latent Symmetry Discovery}
While the main goal of our method is to discover nonlinear symmetries, it can also be adapted for related purposes. We present several use cases of \ours{} as follows.
\vspace{-1mm}
\paragraph{Learning equivariant representation.}
Learning equivariant representation can be viewed as a special case of our method, where the symmetry group $G$ and its representation $\pi$ are known. Our encoder $\phi$ then becomes a $G$-equivariant function in the sense that 
\begin{equation}
    \phi(\pi'(g,x)) = \phi((\psi\circ\pi(g)\circ\phi)(x)) = \pi(g)\phi(x)
\end{equation}
In other words, by fixing $\pi$ to a known group representation, our method learns a $G$-equivariant representation $z=\phi(x)$. Compared to other methods, \ours{} can learn equivariant representation without any knowledge of the group transformation associated with each data sample.

\paragraph{Joint discovery of governing equation.}
\ours{} is analogous to latent space equation discovery techniques \citep{sindyae} in terms of using an autoencoder network for nonlinear coordinate transformations. We can use the latent space learned by \ours{} for discovering equations. Concretely, if we want to find a latent space governing equation parameterized by $\theta$: $\dot z=F_\theta(z)$, where $z=\phi(x)$ is obtained from our encoder network, we fix the encoder $\phi$ and optimize $\theta$ with the objective
$
    l_\text{eq}=\mathbb E_{x,\dot x}\|(\nabla_xz)\dot x-F_\theta(z)\|^2
$.

While equation discovery and symmetry discovery are two seemingly distinct tasks, we will show in the experiment that learning a symmetric latent space can significantly improve the quality of the discovered equation in terms of both its simplicity and its long-term prediction accuracy.

%% file: sec/exp1.tex
\section{Latent Symmetry in Dynamical Systems}\label{sec:exp-dynamics}
In this section, we investigate some dynamical systems with complicated symmetries due to high-dimensional observation space or nonlinear evolution. We show that \ours{} can learn linearized symmetries in the latent space.

\subsection{Datasets}

\paragraph{Reaction-diffusion.}
Many high-dimensional datasets in practical engineering and science problems derive from dynamical systems governed by partial differential equations. These systems often do not exhibit simple linear symmetries in the observation space, but their dynamics might evolve on a low-dimensional manifold with interesting symmetry properties. As an example, we consider a $\lambda-\omega$ reaction-diffusion system \citep{sindyae} governed by
\begin{align*}
    u_t=&(1-(u^2+v^2))u+\beta(u^2+v^2)v+d(u_{xx}+u_{yy})\cr
    v_t=&-\beta(u^2+v^2)u+(1-(u^2+v^2))v+d(u_{xx}+u_{yy})
\end{align*}
with $d=0.1$ and $\beta=1$. We discretize the 2D space into a $100\times100$ grid, which leads to an input dimension of $10^4$. Figure \ref{fig:rd-original} displays a few snapshots of this system. We simulate the system by $6000$ timesteps with step size $0.05$.

The reaction-diffusion system is an example of low-dimensional latent symmetry in high-dimensional observations. In fact, the absence of linear symmetry is not exclusive to high-dimensional systems. We also investigate two low-dimensional dynamics, where their nonlinear evolution prevents any kind of linear symmetry, but our method can still discover meaningful symmetries in the latent space.

\paragraph{Nonlinear pendulum.}
The movement of a simple pendulum can be described by $\dot q=p,\ \dot p=-\omega^2\sin(q)$, with $\omega$ being the natural frequency and $q$ and $p$ the angular displacement and angular momentum. In our experiment, we use $\omega=1$. We simulate $N=200$ trajectories up to $T=500$ timesteps with $\Delta t=0.02$.

\paragraph{Lotka-Volterra System.}
The Lotka-Volterra equations are a pair of nonlinear ODEs that characterize the dynamics of predator-prey interaction. We consider the canonical form of the equations,
$
    \dot p = a - b e ^ q , \ 
    \dot q = c e ^ p - d
$,
where $p$ and $q$ are the logarithm population densities of prey and predator, and the parameters $a, b, c, d$ indicate the growth and death rate of the two populations.
In our experiment, we use $a=2/3, b=4/3$, and $c=d=1$. We simulate $N=200$ trajectories up to $T=10^4$ timesteps with $\Delta t=0.002$.

In the following discussion, we will use $x$ to refer to the states of these systems and $z$ for the latent embeddings of the states. For example, $x = (p, q) ^T \in \R ^ 2$ for the Lotka-Volterra system.

\subsection{Symmetry Discovery}
\begin{figure}[h]
    \centering
    \begin{subfigure}{.23\textwidth}
        \includegraphics[width=\textwidth]{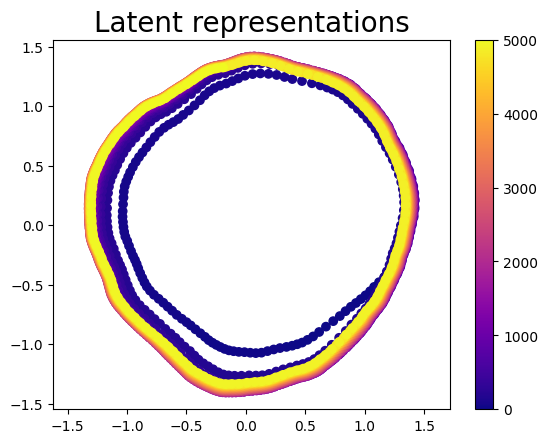}
        \caption{}
        \label{fig:rd-latent-2d}
    \end{subfigure}
    \begin{subfigure}{.184\textwidth}
        \includegraphics[width=\textwidth]{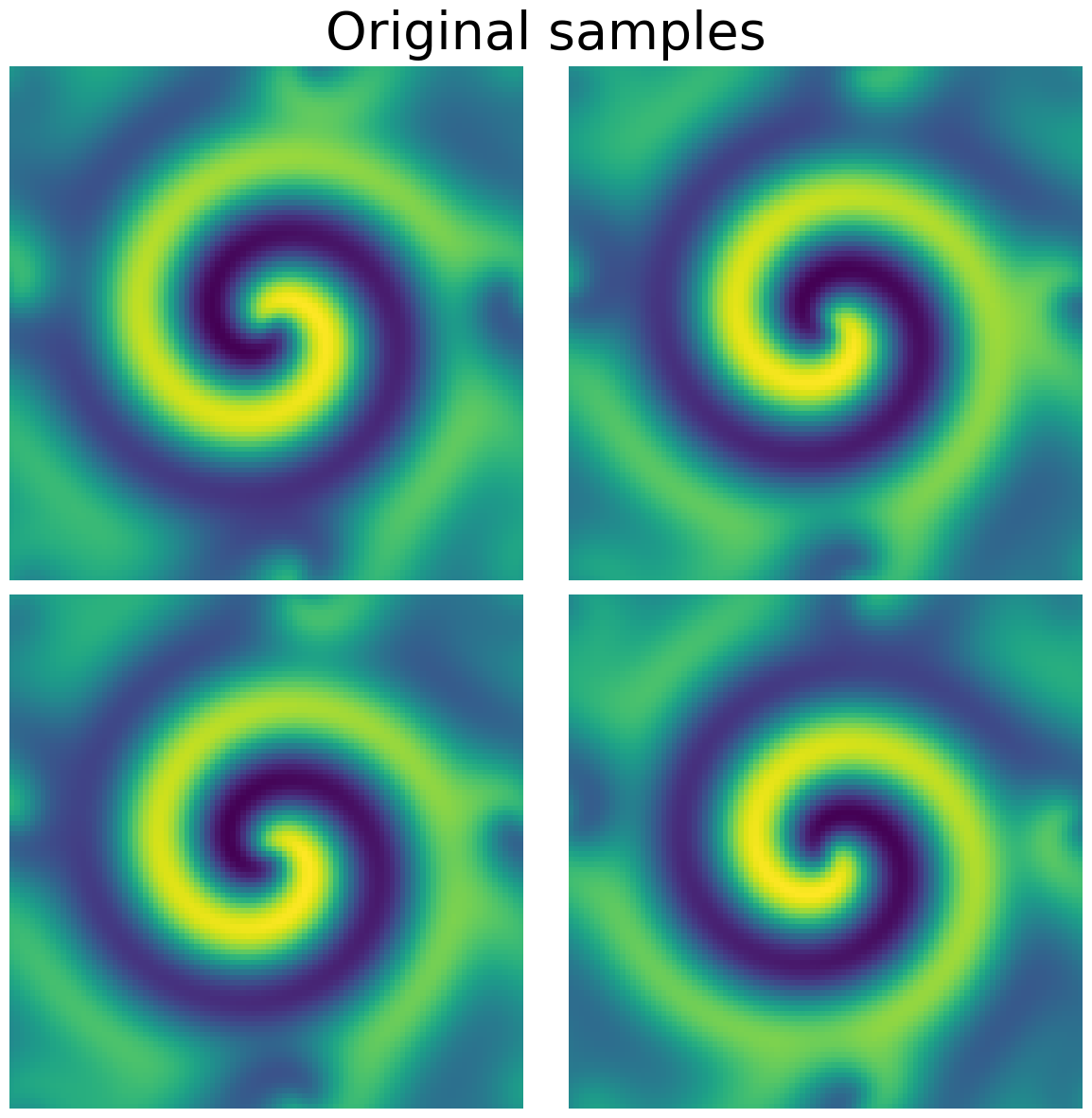}
        \caption{}
        \label{fig:rd-original}
    \end{subfigure}
    \begin{subfigure}{.184\textwidth}
        \includegraphics[width=\textwidth]{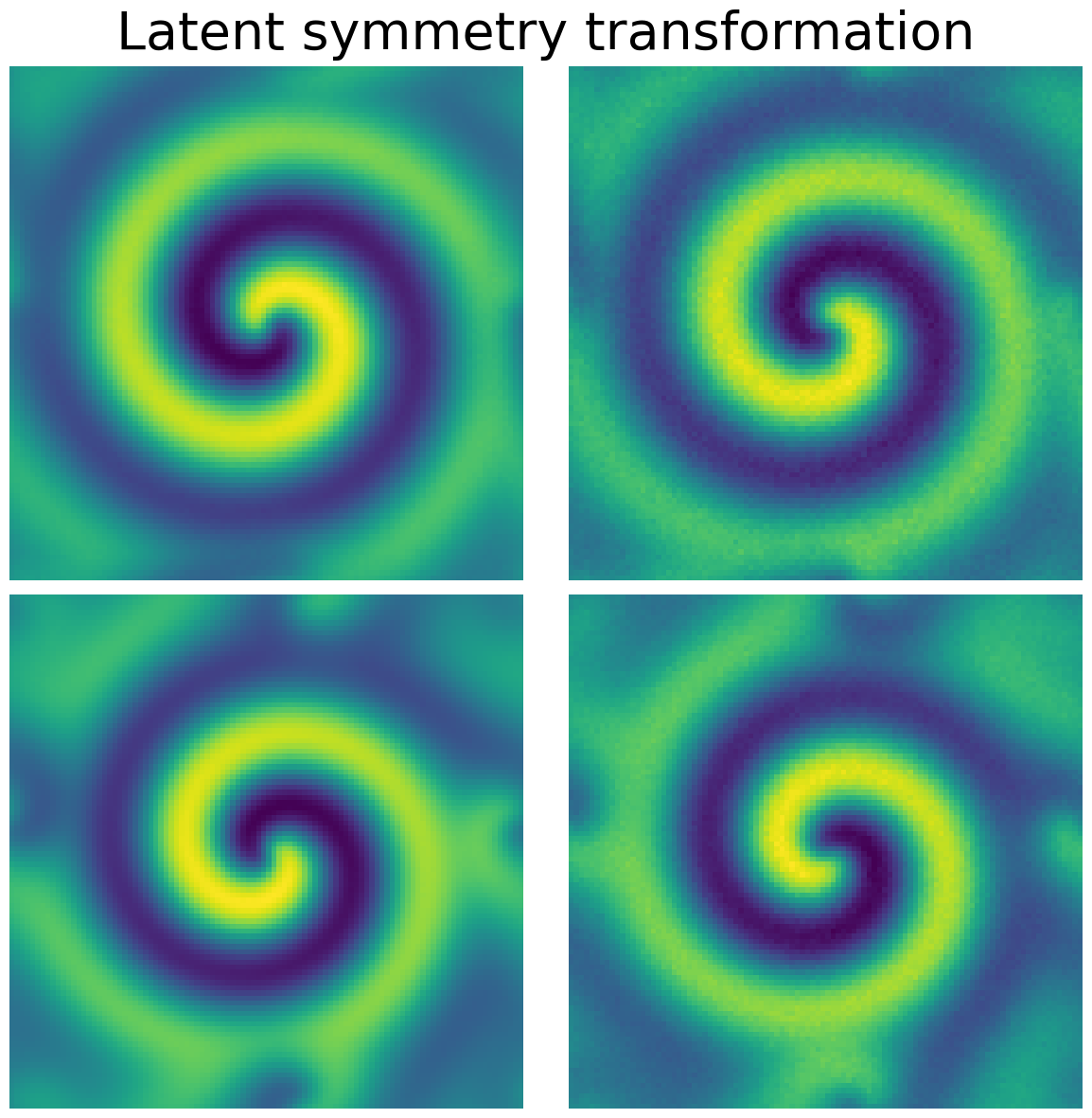}
        \caption{}
        \label{fig:rd-latent-transformed}
    \end{subfigure}
    \hspace{.03\textwidth}
    \begin{subfigure}{.184\textwidth}
        \includegraphics[width=\textwidth]{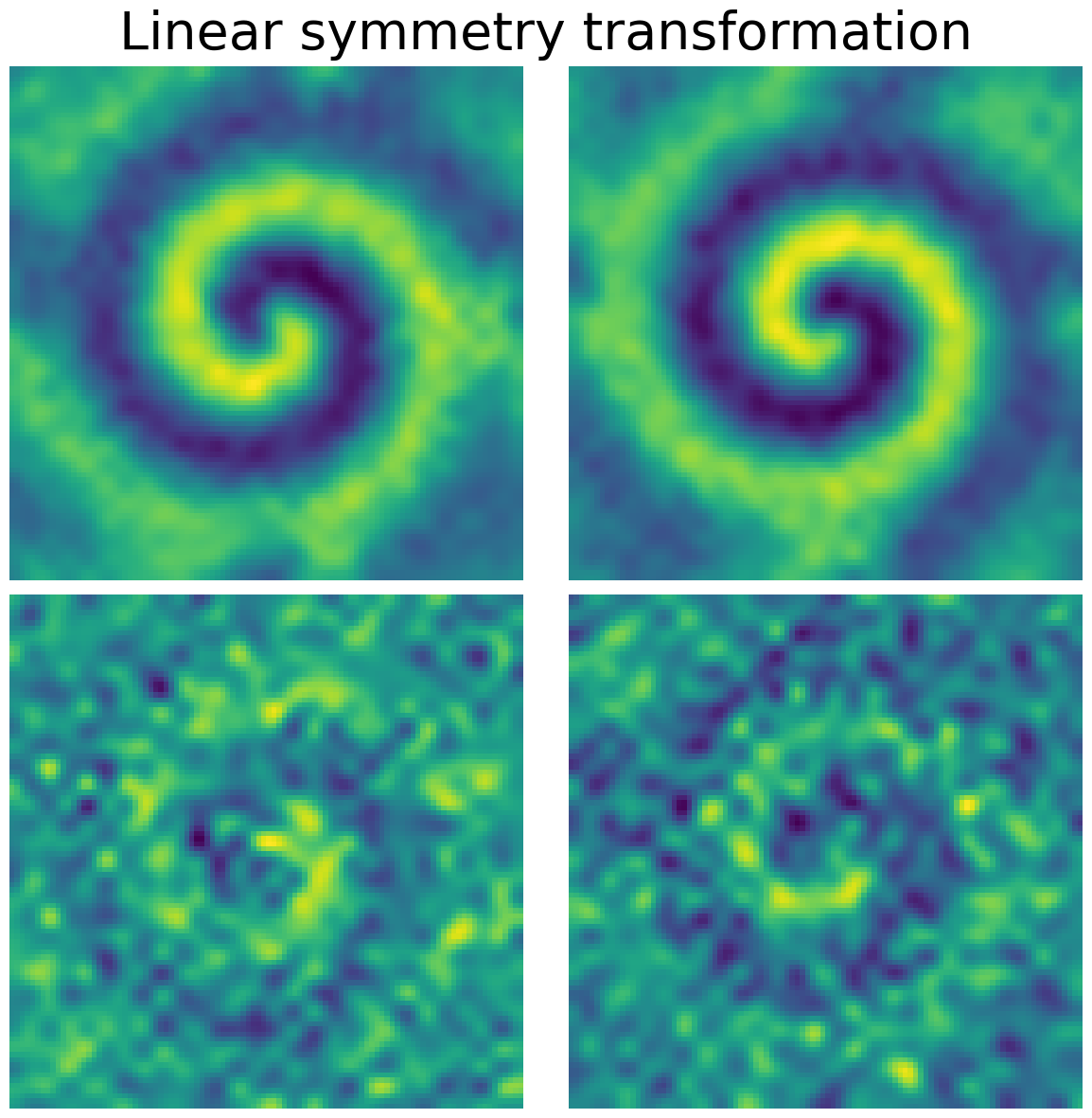}
        \caption{}
        \label{fig:rd-liegan-transformed}
    \end{subfigure}
    \hspace{-0.022\textwidth}
    \caption{Symmetry discovery in reaction-diffusion system with 2D latent space. (a) Latent representations of the system at all timesteps. (b) Randomly selected samples from the dataset. (c) Samples transformed by \ours{} are similar to the original data. (d) Samples transformed by the baseline, linear LieGAN, are significantly different from the original data.}
    \label{fig:rd-2d}
\end{figure}

We train \ours{} to learn the nonlinear mappings between observations and latent representations, along with the linear symmetry in the latent space. We aim to discover the equivariance of latent dynamics, i.e. $z_{t+1}=f(z_t)\Rightarrow gz_{t+1}=f(gz_t)$. Therefore, we take two consecutive timesteps $(x_t, x_{t+1})$ as input, encode them to latent representations with the same encoder weights, and apply the same transformations sampled from the symmetry generator.

For the reaction-diffusion system, we follow the setting in \citet{sindyae} and set the latent dimension $k=2$. Figure \ref{fig:rd-latent-2d} shows how the system evolves in the latent space throughout $T=5000$ timesteps. The Lie algebra basis discovered in the latent space is \(L=\begin{bmatrix}
    0.06 & -3.07 \\ 3.05 & -0.04
\end{bmatrix}\). This suggests an approximate $\mathrm{SO}(2)$ symmetry, which is evident from the visualization.

For the pendulum and the Lotka-Volterra system, we also set the latent dimensions to $2$, which is the same as their input dimensions. \cref{fig:ode-original} shows the trajectories of these two systems in the latent space, with the discovered symmetries
\(L_{\text{pendulum}}=\begin{bmatrix}
    0 & -5.24 \\ 2.16 & 0
\end{bmatrix}\) and \(L_{\text{LV}}=\begin{bmatrix}
    0 & 2.43 \\ -2.74 & 0
\end{bmatrix}\). These indicate rotation symmetries up to a certain scaling in the latent dimensions.


\begin{figure}[ht]
    \centering
    \begin{subfigure}{.11\textwidth}
        \includegraphics[width=\textwidth, viewport=80 50 400 300, clip]{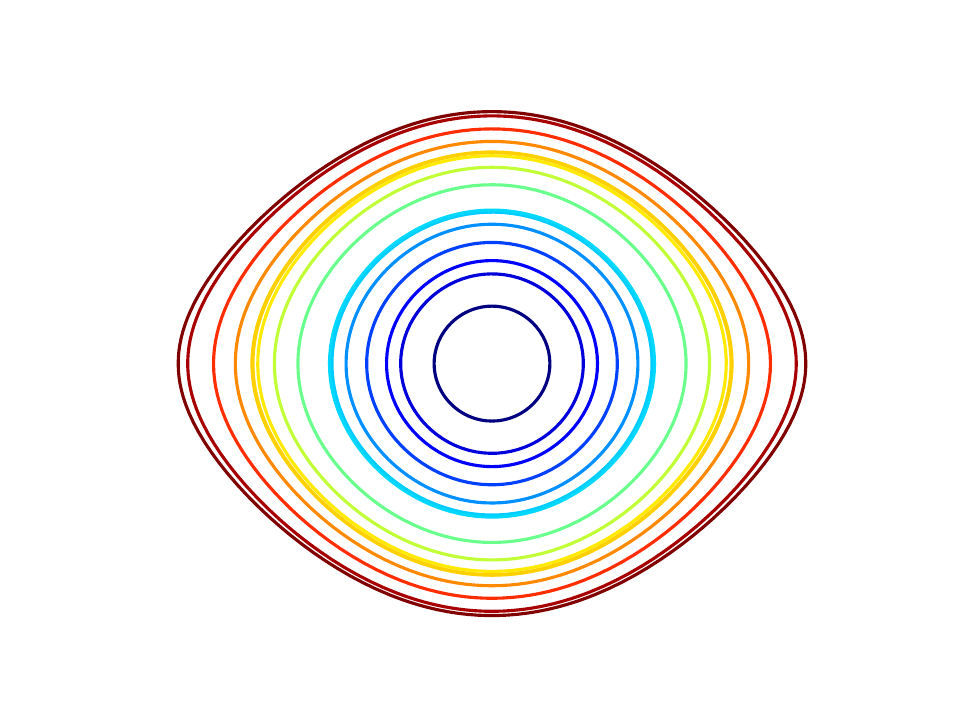}
    \end{subfigure}
    \hfill
    \begin{subfigure}{.11\textwidth}
        \includegraphics[width=\textwidth, viewport=73 50 400 300, clip]{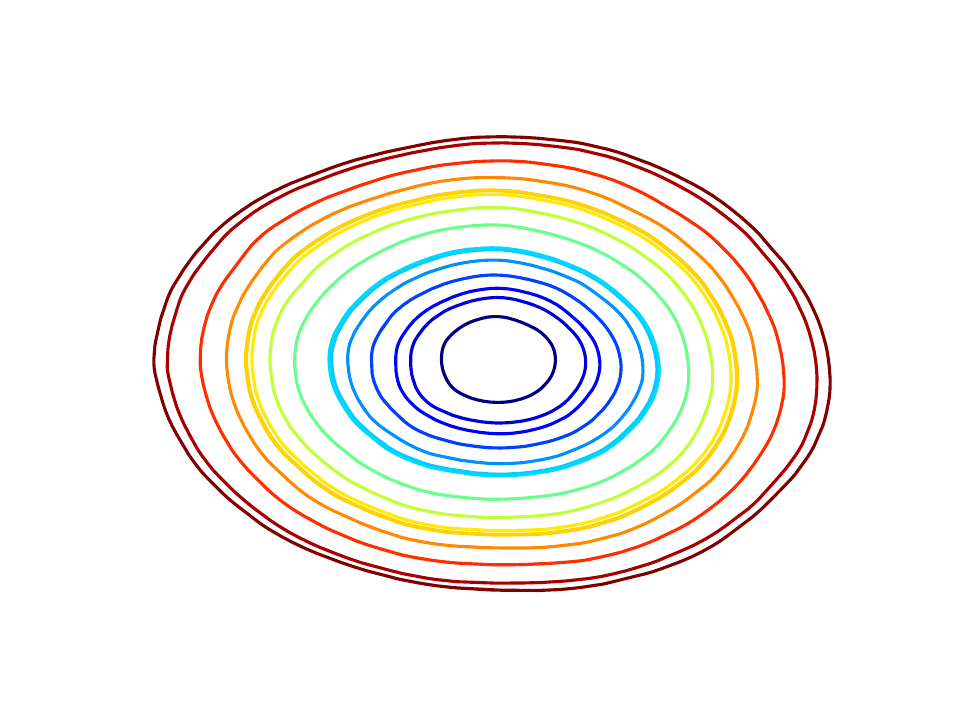}
    \end{subfigure}
    \hfill
    \begin{subfigure}{.11\textwidth}
        \includegraphics[width=\textwidth, viewport=80 50 400 300, clip]{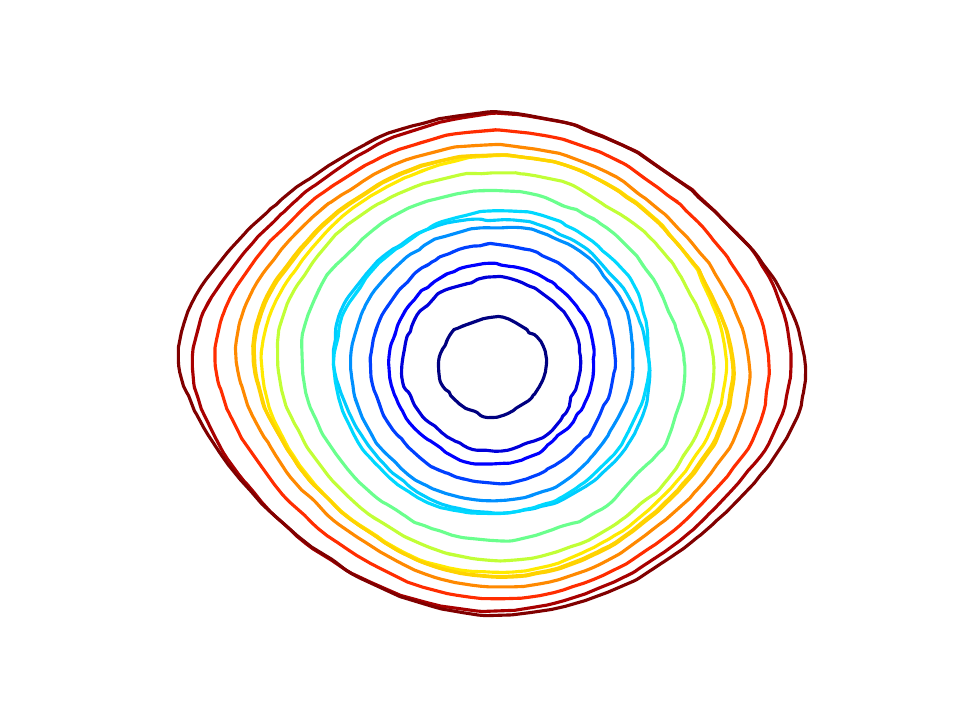}
    \end{subfigure}
    \hfill
    \begin{subfigure}{.11\textwidth}
        \includegraphics[width=\textwidth, viewport=80 50 400 300, clip]{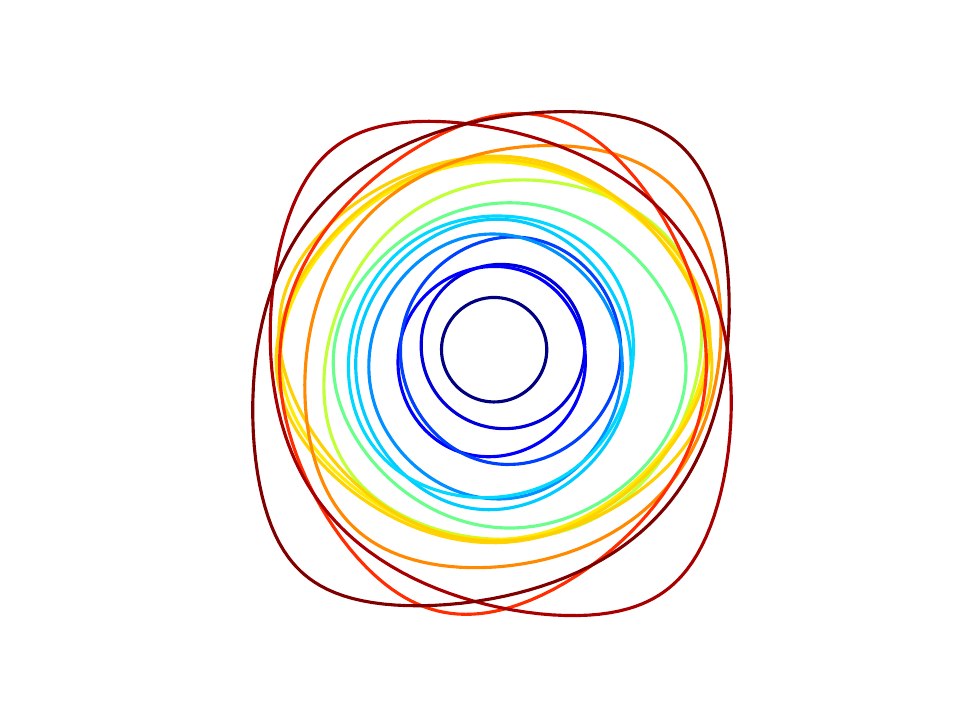}
    \end{subfigure}
    \begin{subfigure}{.11\textwidth}
        \includegraphics[width=\textwidth, viewport=130 50 350 330, clip]{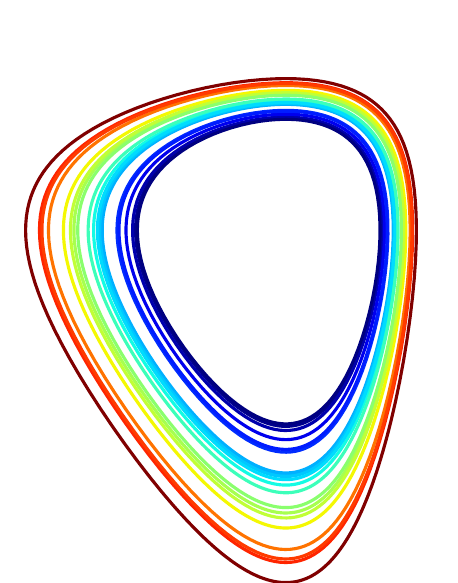}
        \caption{}
        \label{fig:ode-original}
    \end{subfigure}
    \hfill
    \begin{subfigure}{.11\textwidth}
        \includegraphics[width=\textwidth, viewport=110 50 360 330, clip]{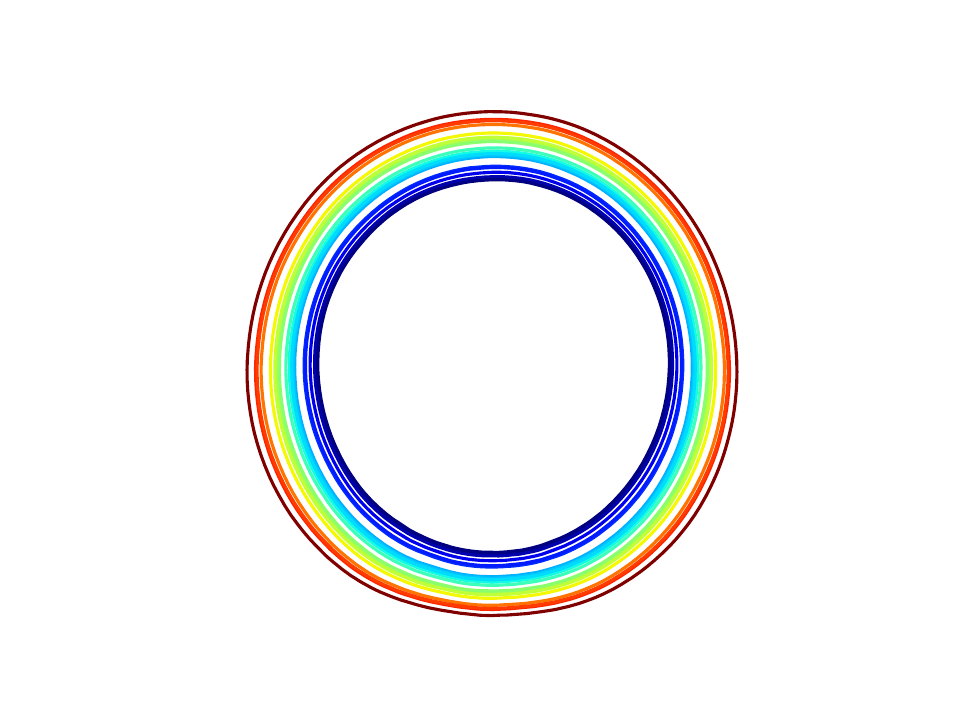}
        \caption{}
        \label{fig:ode-latent}
    \end{subfigure}
    \hfill
    \begin{subfigure}{.11\textwidth}
        \includegraphics[width=\textwidth, viewport=130 50 350 330, clip]{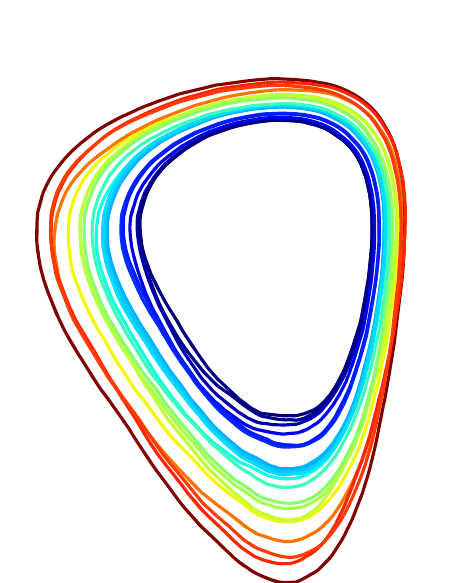}
        \caption{}
        \label{fig:ode-latent-transform}
    \end{subfigure}
    \hfill
    \begin{subfigure}{.11\textwidth}
        \includegraphics[width=\textwidth, viewport=130 50 350 330, clip]{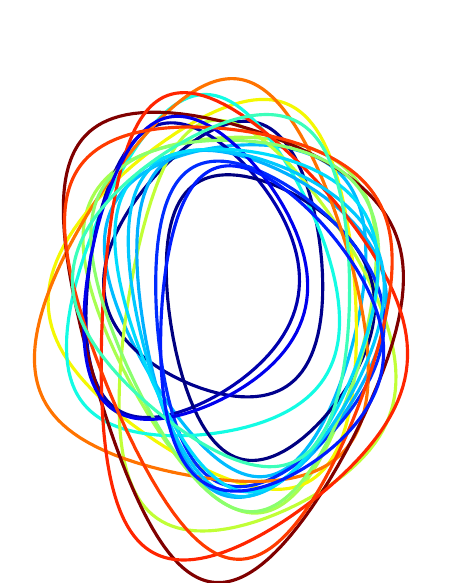}
        \caption{}
        \label{fig:ode-linear-transform}
    \end{subfigure}
    
    \caption{Latent symmetry discovery in nonlinear pendulum (upper) and Lotka-Volterra equations (lower).  (a) Original trajectories of the systems, where the color of each trajectory corresponds to its Hamiltonian. (b) The trajectories mapped to a symmetric latent space. (c) The trajectories transformed by \ours{}. (d) The trajectories transformed by linear LieGAN.}
    \label{fig:sym-dis-nd}
\end{figure}

The accuracy of the discovered symmetry can be verified by visually inspecting the difference between the transformed and the original samples. For the reaction-diffusion system, Figure \ref{fig:rd-latent-transformed} shows some samples with random transformations produced by our method, which are similar to the original data displayed in Figure \ref{fig:rd-original}. We also apply the original LieGAN to this task for comparison, and the transformed samples are shown in Figure \ref{fig:rd-liegan-transformed}. These samples contain obvious artifacts and are noticeably different from the original data, which suggests the necessity of our method when linear symmetry does not exist in observation space.

Similarly, for the pendulum and the Lotka-Volterra system, we use the learned symmetries to transform each entire trajectory, as shown in Figure \ref{fig:ode-latent-transform}. Each trajectory is transformed from the original trajectory of the same color. While each individual data point is taken into a new position, the entire trajectories remain similar before and after transformation, suggesting that the discovered transformations are indeed the symmetries of these systems. In contrast, the linear symmetries learned by LieGAN do not preserve valid trajectories in the observation space, as shown in Figure \ref{fig:ode-linear-transform}.

Besides the visualizations, we evaluate the learned symmetries quantitatively by equivariance error and discriminator logit invariance error \cite{moskalev2023genuine}, defined as
\begin{align}
    EE &= \mathbb E_{x,g} \|f(gx) - gf(x)\| ^2 \label{eq:ee} \\
    DLI &= \mathbb E_{v,g} \frac{1}{2} \| D(v) - D(gv) \| ^2 \label{eq:dli}
\end{align}
where we use $g$ to denote both the group element and its actions, $f$ is the prediction function $x_{t+1}=f(x_t)$, $D$ is the discriminator and $v=(x_t,x_{t+1})$ is the input to \ours{}.
The results are shown in \cref{tab:quantitative}. The learned symmetries from \ours{} achieve lower errors, suggesting that these nonlinear group actions can accurately describe the symmetries of the above systems. A more detailed discussion on how to calculate and interpret these errors is available in Appendix \ref{sec:symm-eval}.

\begin{table}[h]
\small
    \centering
    \begin{tabular}{>{\centering\arraybackslash}p{1.3cm}ccc}
    \hline
        System & Symmetry & Equiv. error & Logit inv. error \\
    \hline
        \multirow{2}{*}{R-D} & \ours{} & \textbf{1.02e-4} & \textbf{2.79e-3} \\
        & LieGAN & - & 3.11e-2 \\
    \hline
        \multirow{2}{*}{L-V} & \ours{} & \textbf{3.00e-2} & \textbf{5.21e-3} \\
        & LieGAN & 8.44e-2 & 4.05e-1 \\
    \hline
        \multirow{2}{*}{Pendulum} & \ours{} & \textbf{4.01e-3} & \textbf{5.33e-3} \\
        & LieGAN & 6.30e-3 & 2.11e-2 \\
    \hline
    \end{tabular}
    \caption{Quantitative metrics for the learned symmetries on test datasets. Equiv. error stands for equivariance error. Logit inv. error stands for logit invariance error. \ours{} can discover nonlinear group actions that more accurately describe the symmetries of the considered dynamical systems. See \cref{sec:symm-eval} for further discussion.}
    \label{tab:quantitative}
\end{table}

\subsection{Effect of Hyperparamemters}
The latent dimension $k$ is a key hyperparameter in our method. 
However, it is not always possible to choose the perfect latent dimension that matches the intrinsic dimension of the system and uncovers symmetry in latent space. To study the robustness of our method under a less ideal configuration, we set the latent dimension $k=3$ for the reaction-diffusion system and repeat the experiment.
As shown in Figure \ref{fig:rd-latent-3d}, the Lie algebra representation is skew-symmetric, indicating rotation symmetry around a particular axis. This can be confirmed as the latent representations roughly dwell on a circular 2D subspace. Although it is not the simplest representation, our method still manages to discover the rotation symmetry as in 2D latent space.

Another hyperparameter that defines the behavior of \ours{} is the dimensionality of the Lie algebra $c$. In the previous experiments, we have set $c=1$. This means that \ours{} can only learn a one-dimensional Lie algebra at a time. Choosing a larger $c$ allows us to discover multiple symmetries simultaneously in the latent space.
As an example, we set the Lie algebra dimensionality to $c=2$ in the Lotka-Volterra system. The result of symmetry discovery is shown in \cref{fig:lv-2d-algebra-main}. The Lie algebra basis $L_1$ and $L_2$ correspond to a scaling symmetry and a rotational symmetry (up to a certain scaling and a tilt angle) in the latent space. In the input space, $L_1$ approximately maps one trajectory to another trajectory with a different Hamiltonian, and $L_2$ takes one point to another within the same trajectory.
This experiment shows that our method can discover symmetry groups of different dimensionalities. More detailed discussion about this experiment can be found in \cref{sec:lv-2d-lie-alg}.

\begin{figure}[ht]
    \centering
    \begin{subfigure}{.3065\textwidth}
        \includegraphics[width=\textwidth]{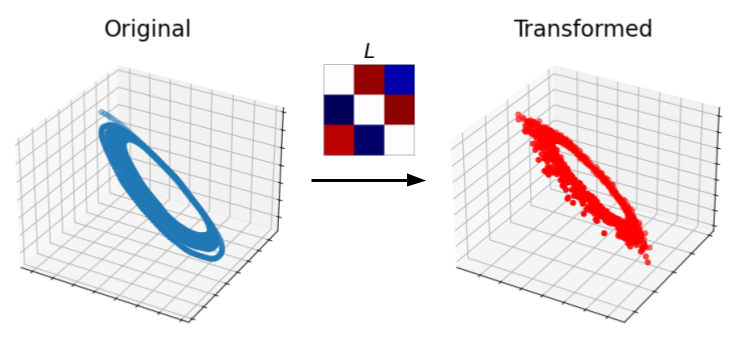}
        \caption{}
        \label{fig:rd-latent-3d}
    \end{subfigure}
    \begin{subfigure}{.13\textwidth}
        \includegraphics[width=\textwidth]{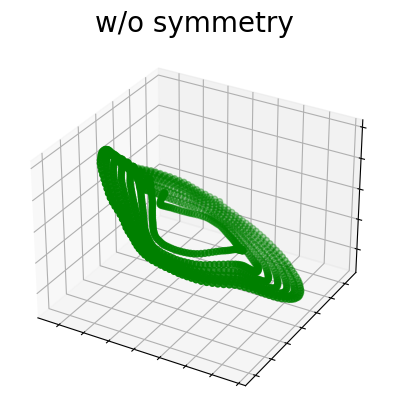}
        \caption{}
        \label{fig:rd-latent-3d-2}
    \end{subfigure}
    \begin{subfigure}{.17\textwidth}
        \includegraphics[width=\textwidth]{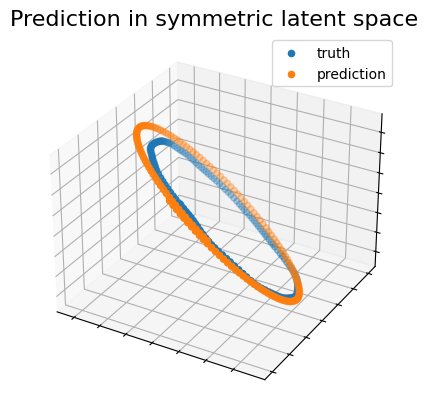}
        \caption{}
        \label{fig:rd-ltp-liegan}
    \end{subfigure}
    \begin{subfigure}{.187\textwidth}
        \includegraphics[width=\textwidth]{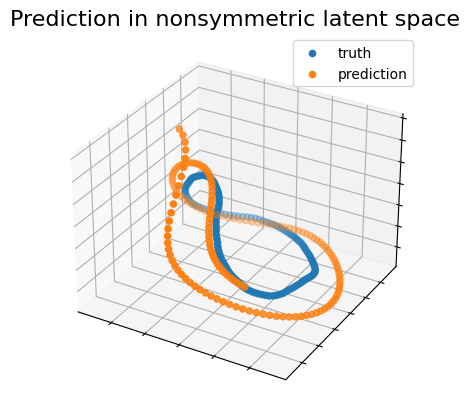}
        \caption{}
        \label{fig:rd-ltp-nosym}
    \end{subfigure}
    \caption{Modeling reaction-diffusion system in 3D latent space. (a) Latent representations before and after our discovered symmetry transformations. (b) The discovered latent space with SINDy but without \ours{}. (c-d) Prediction in both latent spaces.}
    \label{fig:rd-3d}
\end{figure}

\begin{figure}[h]
    \centering
    \includegraphics[width=.48\textwidth, viewport=480 0 1250 270, clip]{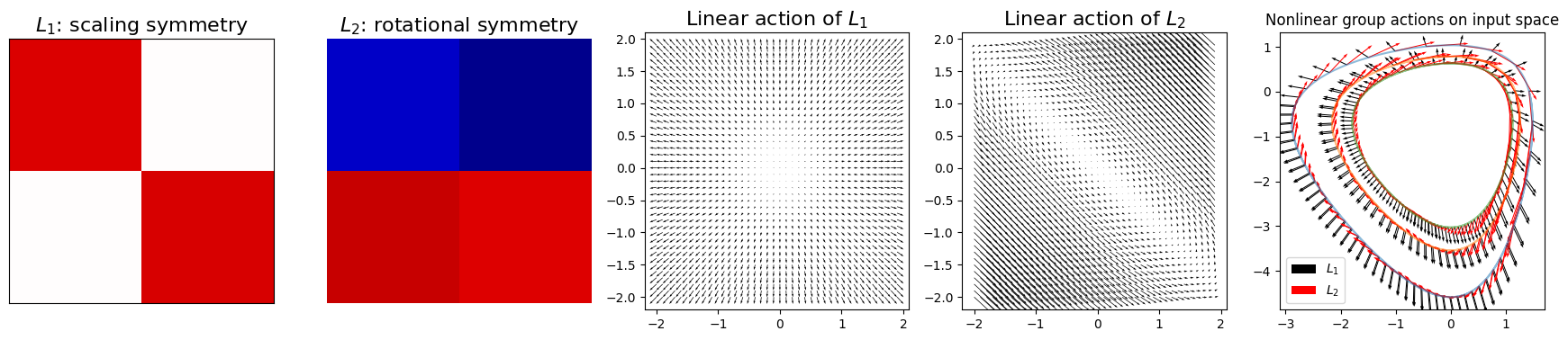}
    \caption{The actions of the discovered 2D Lie algebra on the latent space and the input space.}
    \label{fig:lv-2d-algebra-main}
\end{figure}

\subsection{Using Latent Symmetry for Equation Discovery}\label{sec:eq-dis}

\begin{table}[h]
\scriptsize
    \centering
    \begin{tabular}{ccc}
    \hline
        Model & \ours{} + SINDy & SINDy AE \\
    \hline
        2D & $\begin{aligned}
            \dot z_1=&\quad 0.91z_2\\
            \dot z_2=&-0.91z_1
        \end{aligned}$ & $\begin{aligned}
            \dot z_1=&-0.85z_2\\
            \dot z_2=&\quad0.97z_1
        \end{aligned}$ \\
    \hline 
        3D & $\begin{aligned}
            \dot z_1=&\quad 0.58z_2-0.40z_3\\
            \dot z_2=&-0.56z_1+0.54z_3\\
            \dot z_3=&\quad0.45z_1-0.57z_2
        \end{aligned}$ & $\begin{aligned}
            \dot z_1=&0.65z_2-0.16z_3+\Theta(\mathbf z^2)\\
            \dot z_2=&0.18z_2-0.57z_1+\Theta(\mathbf z^2)\\
            \dot z_3=&0.45z_1-0.57z_2+\Theta(\mathbf z^2)
        \end{aligned}$ \\
    \hline
    \end{tabular}
    
    \caption{Equation discovery on 2D/3D latent spaces for R-D system. Complete results are available in Appendix \ref{sec:supp-rd}.}
    \label{tab:eq-rd}
\end{table}

We demonstrate the benefit of learning latent symmetry by using the latent space to discover governing equations. This is a commonly considered problem in these dynamical systems. We use SINDy \citep{sindy, sindyae} as the equation discovery algorithm, with up to second-order polynomials as candidate functions. The comparison is made between applying SINDy on the latent space learned by our method (\ours{} + SINDy) and using the SINDy autoencoder to learn its own latent space (SINDy AE). The results for the reaction-diffusion system are shown in Table \ref{tab:eq-rd}. The discovered equations from both methods have similar forms in the 2D latent space. In the 3D latent space, the governing equation learned in the \ours{} latent space remains linear. On the other hand, applying the SINDy autoencoder alone results in a nonsymmetric latent space (Figure \ref{fig:rd-latent-3d-2}) and a highly complicated governing equation with second-order terms (Table \ref{tab:eq-rd}).



\paragraph{Long-term forecasting.}

To further verify the accuracy of the discovered equations, we use these equations to simulate the dynamics in the latent space. Concretely, given the initial input frame $x_0$, we obtain its latent representation $\hat z_0=\phi(x_0)$ and predict the future $T$ timesteps by iteratively computing $\hat z_{t+1}=\hat z_t+F(\hat z_t)\cdot\Delta t$, where $\dot z=F(z)$ denotes the discovered governing equation. Then, we map the representations back to the input space by $\hat x_t=\psi(\hat z_t)$.
Figure \ref{fig:rd-ltp-liegan} and \ref{fig:rd-ltp-nosym} show the simulated latent trajectories from the equations discovered in 3D latent space with and without \ours{}. The trajectory remains close to ground truth in the symmetric latent space but diverges quickly for the equation from SINDy AE.
Quantitatively, we also show that the discovered equation in the \ours{} latent space has a lower prediction error. We present the full results in \cref{fig:rd-ltp-mse}, \cref{sec:supp-rd}.

We also conduct the same experiments of equation discovery and long-term forecasting for the nonlinear pendulum and the Lotka-Volterra system. The results are available in Appendix \ref{sec:supp-ode}. While they have simple closed-form governing equations in the observation space, we find that discovering a latent space with learnable symmetry can still be beneficial. The symmetry enforces linear governing equations and reduces error accumulation in long-term forecasting. 

%% file: sec/exp3.tex
\vspace{-2mm}
\section{Learning Equivariant Representation}
\begin{figure}[h]
\vspace{-1mm}
    \centering
    \begin{subfigure}{.18\textwidth}
        \includegraphics[width=\textwidth, viewport=0 0 280 290, clip]{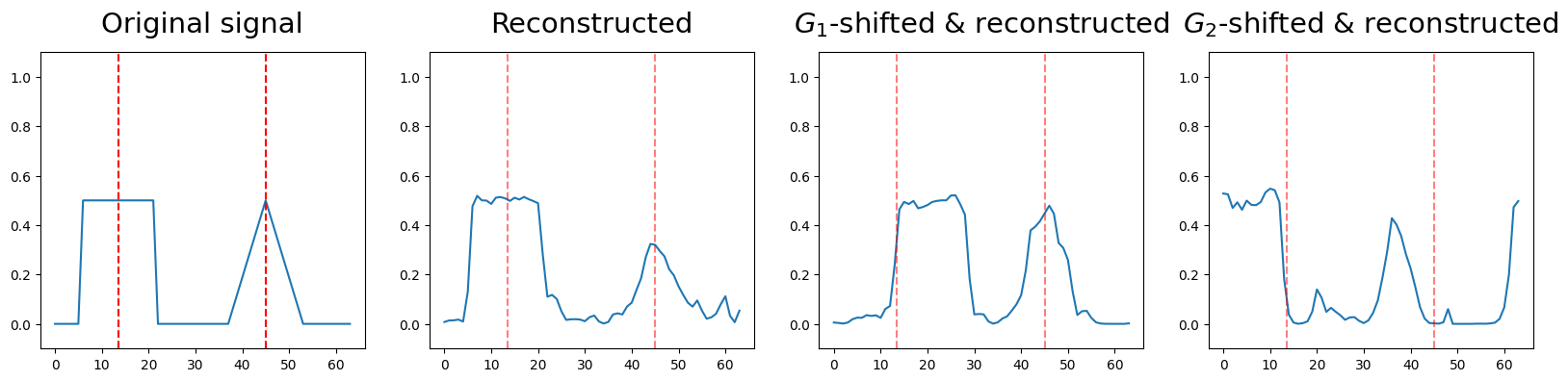}
        \label{fig:db-recon1}
    \end{subfigure}
    \begin{subfigure}{.2\textwidth}
        \includegraphics[width=\textwidth, viewport=0 -100 512 270, clip]{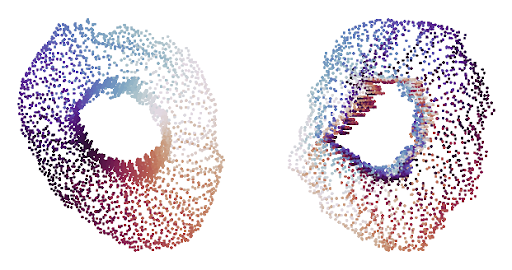}
        \label{fig:db-rep}
    \end{subfigure}
    \vspace{-2mm}
    \caption{Learning equivariant representation of the double-bump world. Left: An original signal $x \in \mathbb R^{64}$. Right: Learned latent space as the direct sum of two 2D subspaces. The color of a data point corresponds to the location of the rectangular bump in the first component and the triangular bump in the second. See \cref{fig:db-full} for full results.}
    \label{fig:db}
\end{figure}

When linear group representation is given, we can use \ours{} for learning the corresponding group equivariant representation. Unlike previous works \citep{garrido2023sie, shakerinava2022structuring}, we learn the representation without any knowledge of the group element associated with each data point. We consider the example of a double-bump world in \citet{shakerinava2022structuring}. It consists of a rectangular and a triangular bump signal, both cyclically shifted in a window. The signal is visualized in \cref{fig:db}. The cyclic translation of each bump forms an $\mathrm{SO(2)}$ group. As each bump is shifted independently, the symmetry group for the composed signal is $\mathrm{SO(2)} \times \mathrm{SO(2)}$. Therefore, we use a 4-dimensional latent space $Z=\mathbb R^2 \oplus \mathbb R^2$ and fix the Lie algebra basis to $L = L_1 \oplus L_2$, \(L_1 = L_2 = \begin{bmatrix}
    0 & 1 \\ -1 & 0
\end{bmatrix}\).

Figure \ref{fig:db} (right) shows the latent space learned by \ours{}. We observe that rotation in the first component shifts the rectangular bump, while rotation in the second component simultaneously shifts both bumps. We provide a more detailed discussion in \cref{sec:more-exp-db} with additional visualizations of transformed and reconstructed samples. This is an example that how our method can learn equivariant representations when we do not know the group transformation of each data point. We also include another experiment on $\mathrm{SO}(3)$ equivariant representation for a 3D object in Appendix \ref{sec:supp-rot3d}.

%% file: sec/conclusion.tex
\section{Discussion}

We propose \ours{}, a novel generative modeling framework for discovering nonlinear symmetries. \ours{} decomposes the group action as a linear representation on a latent space and a pair of nonlinear mappings between the latent space and the observation space. By jointly optimizing the group representation and the nonlinear mappings, it discovers both the symmetry group and its nonlinear group action on the data. We also show that it can be applied to downstream tasks such as equation discovery, leading to simpler equations and better long-term prediction accuracy.

A limitation of our work lies in \cref{thm:uanga}, which only guarantees that our method can model actions of compact groups, among other restrictions. However, the results in \cref{sec:lv-2d-lie-alg} and \ref{sec:top-tagging} suggest that noncompact symmetry groups can also be learned. Thus, an important direction for future work is to develop the theory for modeling more general group actions within our proposed framework. We also plan to investigate the connection between symmetry and other physical properties such as conservation laws. Given the prevalence of symmetries in the natural world, our long-term goal is to develop a general framework for automatically discovering symmetries and other types of governing laws from data and accelerate scientific discovery.


%% file: appendix/more_exp.tex
\section{Supplementary Experiment Results}
\subsection{High-Dimensional Reaction-Diffusion System}\label{sec:supp-rd}

\begin{table}[h]
    \small
    \centering
    \begin{tabular}{l|l}
    \hline
        Model & Discovered equation \\
    \hline
        \ours{} + SINDy & $\begin{aligned}
            \dot z_1=&\quad 0.43z_2-0.53z_3\\
            \dot z_2=&-0.51z_1+0.66z_3\\
            \dot z_3=&\quad0.47z_1-0.52z_2
        \end{aligned}$\\
    \hline 
        \ours{} + SINDy + PCA & 
        $\begin{aligned}
            \dot u_1=&-0.98u_2\\
            \dot u_2=&\quad 0.84u_1\\
            \dot u_3=&\quad0
        \end{aligned}$\\
    \hline 
        SINDy AE & 
        $\begin{aligned}
            \dot z_1=&\quad0.65z_2-0.16z_3+0.20z_1^2+0.11z_1z_2+0.29z_1z_3\\
            &-0.41z_2z_3-0.16z_3^2\\
            \dot z_2=&-0.57z_1+0.18z_2-0.24z_1z_2+0.46z_1z_3-0.18z_2^2\\
            &-0.26z_2z_3+0.29z_3^2\\
            \dot z_3=&\quad0.45z_1-0.57z_2-0.27z_1^2+0.18z_2^2-0.19z_2z_3
        \end{aligned}$\\
    \hline
        SINDy AE + PCA & 
        $\begin{aligned}
            \dot u_1=&\quad0.95u_2-0.06u_3+0.09u_1u_2+0.16u_1u_3-0.59u_2u_3-0.12u_3^2\\
            \dot u_2=&-0.58u_1+0.29u_3-0.57u_1u_3-0.23u_2u_3-0.10u_3^2\\
            \dot u_3=&-0.06u_1^2+0.51u_1u_2+0.08u_2^2+0.35u_2u_3
        \end{aligned}$\\
    \hline
    \end{tabular}
    \caption{Complete equation discovery results on 3D latent space for reaction-diffusion system.}
    \label{tab:eq-rd-3d}
\end{table}

\begin{wrapfigure}{r}{5.5cm}
    \vspace{-3mm}
    \centering
        \includegraphics[width=\linewidth]{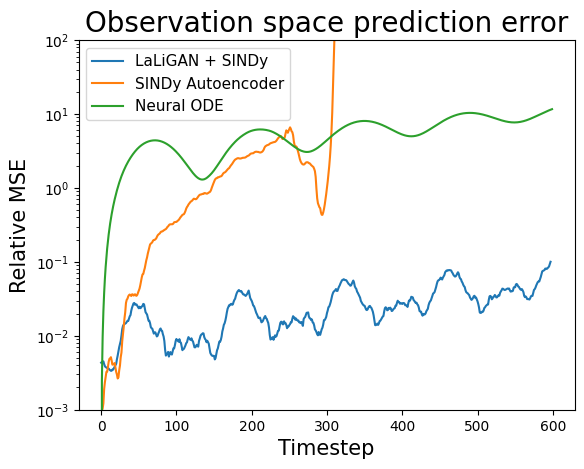}
    \caption{Relative MSE losses of long-term simulations of R-D system.}
    \label{fig:rd-ltp-mse}
    \vspace{-2mm}
\end{wrapfigure}

Table \ref{tab:eq-rd-3d} expands the folded results in Table \ref{tab:eq-rd} of discovering governing equations in 3D latent space for reaction-diffusion system. Applying the SINDy autoencoder alone results in a highly complicated governing equation with several second-order terms.
Also, we find that the equations learned on the symmetric latent space can even be further simplified with a linear transformation introduced by principle component analysis (PCA). The $u$'s in the equations denote the principle components, sorted by the variance each component explains. In comparison, the equations from the SINDy autoencoder alone do not admit a simpler form under the linear PCA transformation.

We also evaluate the forecasting accuracy quantitatively by the relative MSE between the prediction and ground truth in the observation space, as is shown in Figure \ref{fig:rd-ltp-mse}. Besides the symbolic models in Table \ref{tab:eq-rd-3d}, we also include Neural ODE \citep{node} as a baseline. Similar to the symbolic equation discovery, it can also predict the dynamics at arbitrary timesteps with an ODE parametrized by a neural network. Figure \ref{fig:rd-ltp-mse} shows that the discovered equation learned with latent space symmetry outperforms both the equation from vanilla SINDy AE and the Neural ODE model in this task of long-term dynamics forecasting.

\newpage
\subsection{Nonlinear Ordinary Differential Equations}\label{sec:supp-ode}
\begin{table}[h]
\small
    \centering
    \begin{tabular}{l|l}
    \hline
        Method & Discovered equation \\
    \hline 
        \ours{} + SINDy & $\begin{aligned}
            \dot z_1=& -0.94z_2\\
            \dot z_2=&\quad0.38z_1
        \end{aligned}$ \\
    \hline
        SINDy & $\begin{aligned}
            \dot z_1=& \quad 0.99z_2\\
            \dot z_2=&-0.98 \sin (z_1)
        \end{aligned}$\\
    \hline 
        SINDy AE & $\begin{aligned}
            \dot z_1=& -0.46\sin(z_2)\\
            \dot z_2=&\quad 0.51 z_1 + 0.42 \sin(z_1)
        \end{aligned}$\\
    \hline
    \end{tabular}
    \caption{Equation discovery for pendulum.}
    \label{tab:eq-pd}

\end{table}
\begin{table}[h]
\small
    \centering
    \begin{tabular}{l|l}
    \hline
        Method & Discovered equation \\
    \hline 
        \ours{} + SINDy & $\begin{aligned}
            \dot z_1=& -0.65 - 0.56z_2\\
            \dot z_2=& -0.14 + 0.67z_1
        \end{aligned}$ \\
    \hline
        SINDy & $\begin{aligned}
            \dot z_1=& \quad0.64 - 1.28 e^{z_2}\\
            \dot z_2=& -0.91 + 1.05 e^{z_1}
        \end{aligned}$ \\
    \hline 
        SINDy AE & $\begin{aligned}
            \dot z_1=& \quad12.47 - 5.27 z_1 + 40.00 z_2 \\ +& 0.19 z_1z_2 - 0.64 z_1^2 - 0.93 e^{z_1}\\
            \dot z_2=& -6.91 - 0.65 z_1^2
        \end{aligned}$\\
    \hline
    \end{tabular}
    \caption{Equation discovery for L-V system.}
    \label{tab:eq-lv}
\end{table}

Table \ref{tab:eq-pd} and \ref{tab:eq-lv} show the equation discovery results for the nonlinear pendulum and the Lotka-Volterra system. For each dataset, we apply three methods for equation discovery: 1) learning a symmetric latent space with \ours{}, and training SINDy with the fixed latent space; 2) training SINDy in the original observation space; 3) training SINDy autoencoder to learn a latent space without symmetry and discover the equation. Unlike the experiment in the high-dimensional reaction-diffusion system, we include SINDy without autoencoder because the observation space is low-dimensional in each of these systems and there indeed exists a closed-form governing equation.

It can be observed that applying \ours{} with SINDy still leads to simple linear equations. For the Lotka-Volterra system, the equation also consists of constant terms because the latent space is not centered at the origin. On the other hand, SINDy almost recovers the ground truth equations in both tasks, with no additional or missing terms but only small numerical errors for the coefficients.

\begin{figure}[ht]
    \centering
    \begin{subfigure}{.35\textwidth}
        \includegraphics[width=\textwidth]{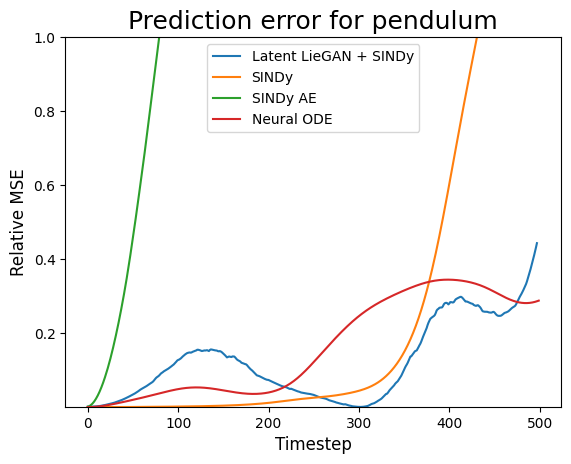}
        \caption{Pendulum}
        \label{fig:pd-ltp-mse}
    \end{subfigure}
    \begin{subfigure}{.35\textwidth}
        \includegraphics[width=\textwidth]{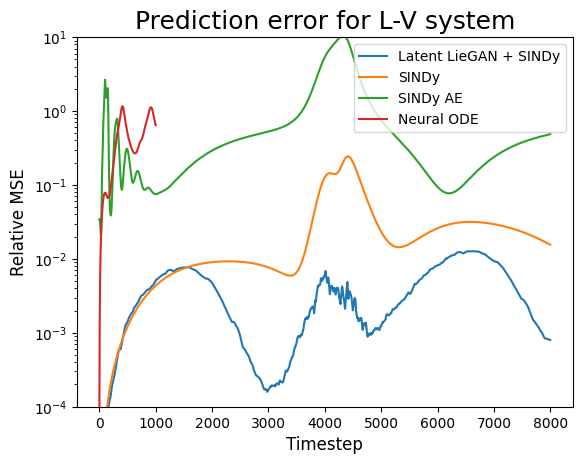}
        \caption{Lotka-Volterra system}
        \label{fig:lv-ltp-mse}
    \end{subfigure}
    \caption{Long-term prediction for nonlinear ODEs.}
    \label{fig:ode-ltp}
\end{figure}

However, our method can still achieve better long-term forecasting accuracy with the discovered equation. Similar to section \ref{sec:eq-dis}, given an initial input frame $x_0$, we obtain its latent representation $\hat z_0=\phi(x_0)$ and predict the future $T$ timesteps by iteratively calculating $\hat z_{t+1}=\hat z_t+F(\hat z_t)\cdot\Delta t$, where $\dot z=F(z)$ denotes the discovered governing equation. Then, we map the representations back to the input space by $\hat x_t=\psi(\hat z_t)$ to get the prediction in observation space. Figure \ref{fig:ode-ltp} shows the relative mean square error at different timesteps. The curve for Neural ODE in the Lotka-Volterra system is incomplete because the prediction goes to NaN after about 1000 steps. Generally, our method leads to the slowest error accumulation. By contrast, while SINDy managed to recover the almost correct equation, the small numerical error can still lead to a very large error after a certain time period.

\subsection{Multi-Dimensional Lie Algebra}\label{sec:lv-2d-lie-alg}

\begin{figure}[h]
    \centering
    \includegraphics[width=.95\textwidth]{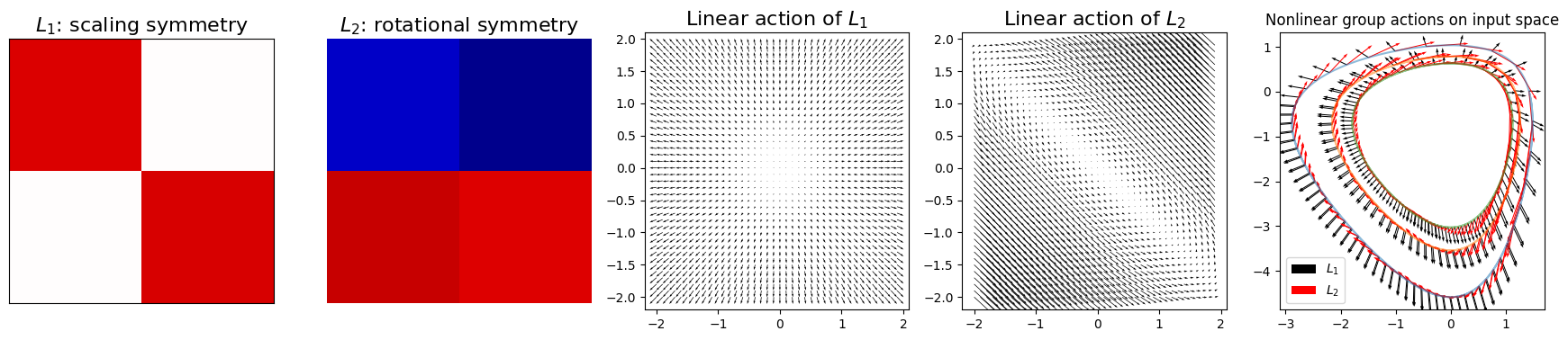}
    \caption{Discovered 2D Lie algebra and its actions on the latent space and the input space.}
    \label{fig:lv-2d-algebra}
\end{figure}

\begin{figure}[h]
    \centering
    \begin{subfigure}{.3\textwidth}
        \includegraphics[width=\textwidth]{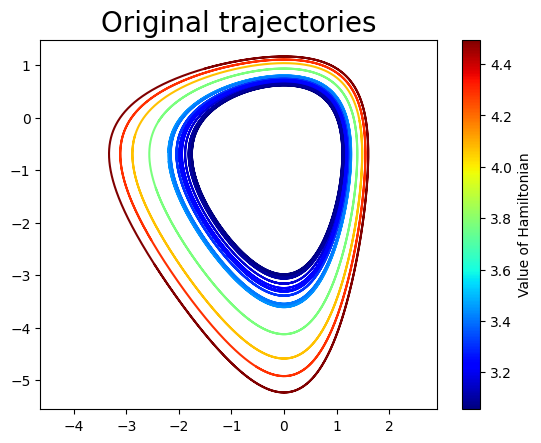}
    \end{subfigure}
    \begin{subfigure}{.3\textwidth}
        \includegraphics[width=\textwidth]{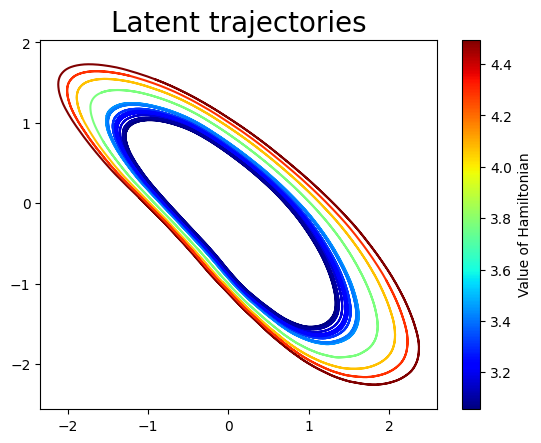}
    \end{subfigure}
    \begin{subfigure}{.3\textwidth}
        \includegraphics[width=\textwidth]{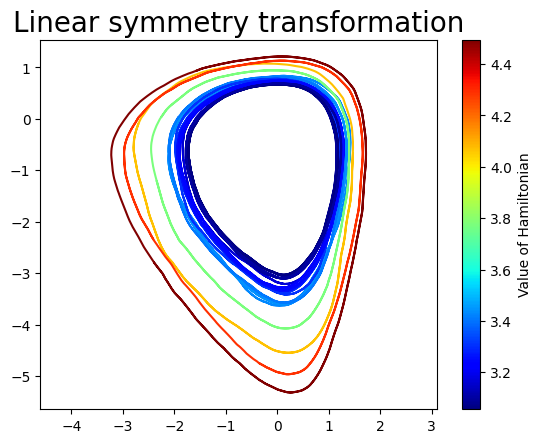}
    \end{subfigure}
    \caption{Left: Original trajectories of the Lotka-Volterra system. Middle: The trajectories mapped to the discovered latent space. Right: Original trajectories transformed by \ours{} with 2D Lie algebra.}
    \label{fig:lv-2d-traj}
\end{figure}

Our method learns a Lie algebra basis $\{ L_i \in \mathbb R^{k \times k} \}_{i=1}^c$. The dimensionality of the Lie algebra, $c$, is a hyperparameter. In the previous experiments, we have set $c=1$, meaning that \ours{} is only learning a one-dimensional Lie algebra at a time. Choosing a different $c$ allows us to discover multiple symmetries simultaneously in the latent space. We demonstrate this with the Lotka-Volterra equation experiment.

In this experiment, we set the latent dimension to 2 and increase the Lie algebra dimension from 1 to 2. Figure \ref{fig:lv-2d-algebra} shows the discovered Lie algebra basis, $L_1$ and $L_2$. One can verify that $\{ L_1, L_2 \}$ forms a valid Lie algebra basis that is closed under the Lie bracket. The actions of $L_1$ and $L_2$ in the latent space can be visualized by the vector fields $(L_1z)^i\partial_i$ and $(L_2z)^i\partial_i$. It can be observed that $L_1$ corresponds to a scaling symmetry and $L_2$ corresponds to a rotational symmetry (up to a certain scaling and a tilt angle). The actions of of $L_1$ and $L_2$ in the input space can be visualized by the vector fields $(\frac{\partial\psi}{\partial z})^{ij}(L_1z)_j\partial_i$ and $(\frac{\partial\psi}{\partial z})^{ij}(L_2z)_j\partial_i$. The rightmost plot shows these vector fields evaluated on the original trajectories.

It is easier to interpret the meaning of these discovered symmetries by looking at the latent trajectories in Figure \ref{fig:lv-2d-traj}. The scaling symmetry $L_1$ changes the Hamiltonian of the system and indicates that the governing equation of the system $z_{t+1}=f(z_t)$ does not change with the Hamiltonian. The rotational symmetry $L_2$ is similar to the original experiment with only one-dimensional Lie algebra, which approximately takes one point to another within the same trajectory. Its representation differs from the previous one-dimensional experiment because the latent embeddings of the trajectories have also changed. Still, it can be interpreted as a time translation symmetry of the system.

\subsection{Lorentz Symmetry in Top Tagging}\label{sec:top-tagging}

We consider the Top Tagging dataset \citep{kasieczka2019machine}, which is also studied in \citet{liegan}. The task is a binary classification between top quark jets and the background signals. There are 2M observations in total, each consisting of the four-momentum of up to 200 particle jets.

This classification task is invariant to the restricted Lorentz group $\mathrm{SO}^+(1,3)$. It is a 6-dimensional Lie group, including the spatial rotations around three axes and the boosts along three spatial directions.

The original dataset has a linear symmetry in the input space. To test whether \ours{} can learn nonlinear group actions, we transform the original inputs to a high-dimensional space and use it as the new input space for \ours{}. Concretely, we choose 4 spatial modes $\mathbf u_i \in \mathbb R^{128}$ given by Legendre polynomials and define $\mathbf u = \sum_{i=1}^4x_i\mathbf u_i$ where $\mathbf x = (x_1, x_2, x_3, x_4)$ is the 4-momentum from the original dataset.

In our experiment, we set the latent dimension to 4 and the Lie algebra dimension to 6. Figure \ref{fig:top-tagging} shows the discovered Lie algebra and its structure constants. Its representation does not match Figure 5 from \citet{liegan}, because the latent representations obtained by the encoder are different from the original 4-momentum inputs. However, we can compute the structure constants of this Lie algebra, which reveal its similar algebraic structure to the ground truth Lorentz algebra $\mathfrak{so}(1,3)$.

\begin{figure}[h]
    \centering
    \begin{subfigure}{.7\textwidth}
        \includegraphics[width=\textwidth]{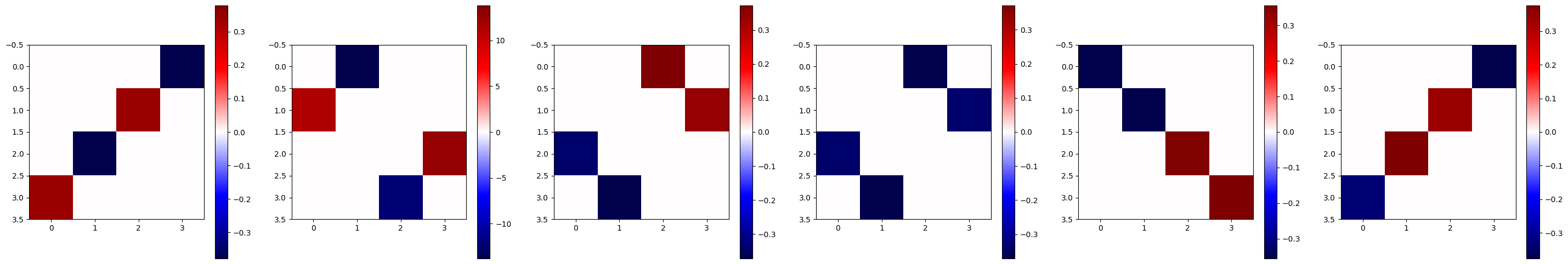}
    \end{subfigure}
    \begin{subfigure}{.13\textwidth}
        \includegraphics[width=\textwidth]{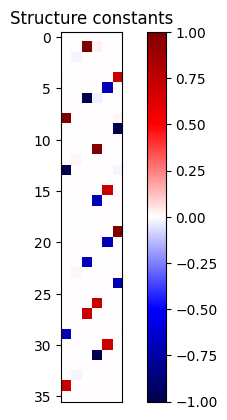}
    \end{subfigure}
    \begin{subfigure}{.117\textwidth}
        \includegraphics[width=\textwidth]{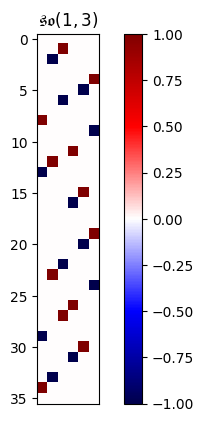}
    \end{subfigure}
    \caption{Left: Discovered 6-dimensional Lie algebra. Right: The structure constants of the discovered Lie algebra and the ground truth $\mathfrak{so}(1,3)$.}
    \label{fig:top-tagging}
\end{figure}

\subsection{Learning $\mathrm{SO}(2) \times \mathrm{SO}(2)$ Equivariant Representation}\label{sec:more-exp-db}

\begin{figure*}[h]
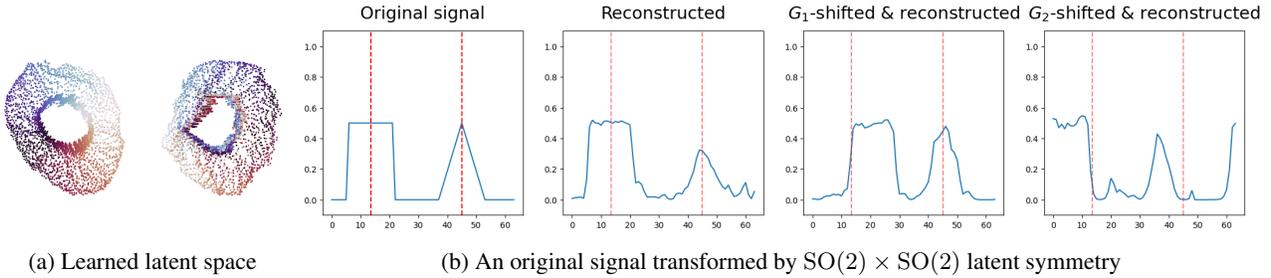

    \centering
    \begin{subfigure}{.23\textwidth}
        \includegraphics[width=\textwidth, viewport=0 -50 512 270, clip]{figs/db/rep.png}
        \caption{Learned latent space}
        \label{fig:db-rep-full}
    \end{subfigure}
    \begin{subfigure}{.75\textwidth}
        \includegraphics[width=\textwidth]{figs/db/recon.png}
        \caption{An original signal transformed by $\mathrm{SO}(2) \times \mathrm{SO}(2)$ latent symmetry}
        \label{fig:db-recon-full}
    \end{subfigure}
    \caption{Learning equivariant representation of the double-bump world. (a) Learned latent space as the direct sum of two 2D subspaces. The color of a data point corresponds to the location of the rectangular bump in the first component and the triangular bump in the second. (b) From left to right: (1) an original signal $x \in \mathbb R^{64}$; (2) reconstructed signal $\psi(\phi(x))$; (3-4) reconstructed signals from transformed latent representations, $\psi((\pi(\theta_1)\oplus I)\phi(x))$ and $\psi((I\oplus \pi(\theta_2))\phi(x))$. The red lines are the bump centers in the original signal.}
    \label{fig:db-full}
    \vspace{-3mm}
\end{figure*}

When linear group representation is given, we can use \ours{} for learning the corresponding group equivariant representation. We consider the example of a double-bump world in \cite{shakerinava2022structuring}. It consists of a rectangular and a triangular bump signal of length 16, both cyclically shifted in a window of length 64. The signal is visualized in \cref{fig:db}. The cyclic translation of each bump forms an $\mathrm{SO(2)}$ group. As each bump is shifted independently, the symmetry group for the composed signal is $\mathrm{SO(2)} \times \mathrm{SO(2)}$. Therefore, we use a 4-dimensional latent space $Z=\mathbb R^2 \oplus \mathbb R^2$ and fix the Lie algebra basis to $L = L_1 \oplus L_2$, $L_1 = L_2 = [0,\ 1;\ -1,\ 0]$.

Figure \ref{fig:db-full} also shows the latent space learned by \ours{}. We observe that rotation in the first component shifts the rectangular bump, while rotation in the second component simultaneously shifts both bumps. This is also evident from the transformed and reconstructed samples in Figure \ref{fig:db-recon-full}. This is an example that our method can learn equivariant representations when we do not know the group transformation of each data point.

\subsection{Learning $\mathrm{SO}(3)$ Equivariant Representation}\label{sec:supp-rot3d}

\begin{figure}[h]
    \centering
    \begin{subfigure}{.15\textwidth}
        \includegraphics[width=\textwidth]{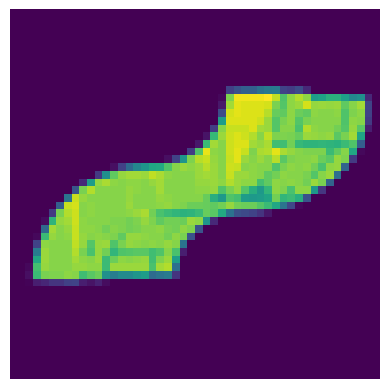}
    \end{subfigure}
    \begin{subfigure}{.15\textwidth}
        \includegraphics[width=\textwidth]{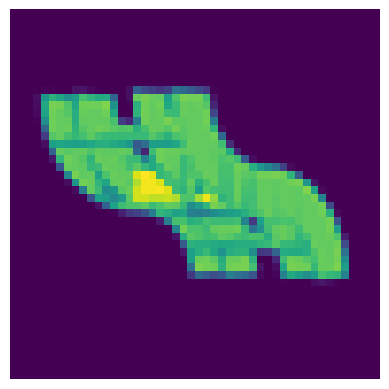}
    \end{subfigure}
    \begin{subfigure}{.15\textwidth}
        \includegraphics[width=\textwidth]{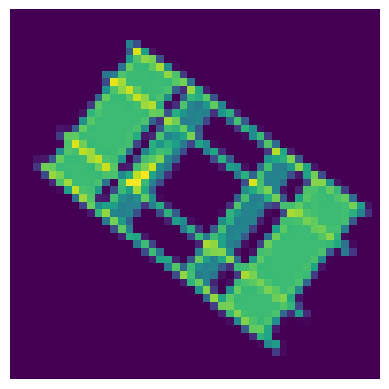}
    \end{subfigure}
    \hfill
    \begin{subfigure}{.15\textwidth}
        \includegraphics[width=\textwidth]{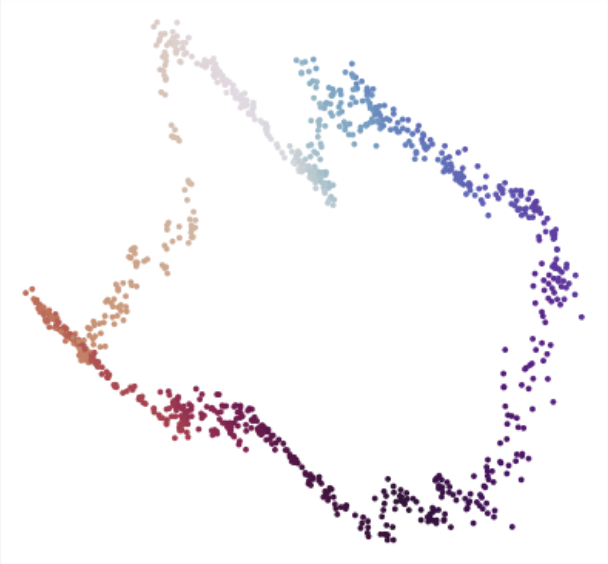}
    \end{subfigure}
    \begin{subfigure}{.15\textwidth}
        \includegraphics[width=\textwidth]{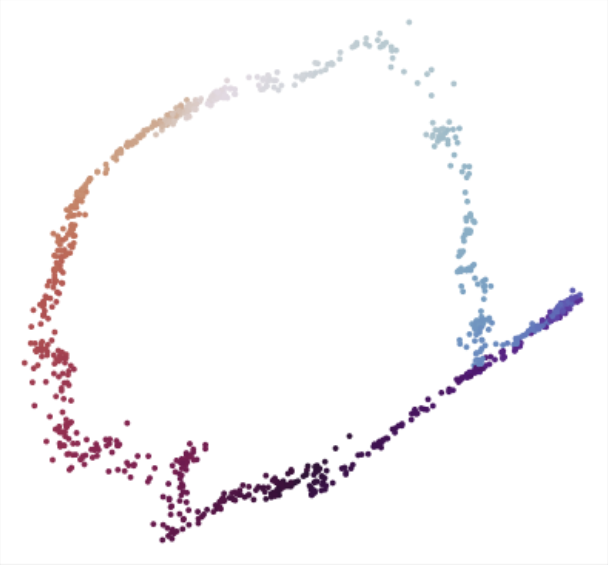}
    \end{subfigure}
    \begin{subfigure}{.15\textwidth}
        \includegraphics[width=\textwidth]{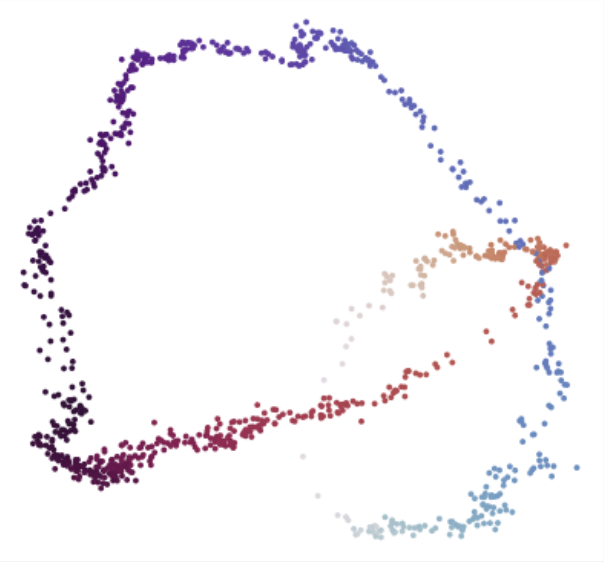}
    \end{subfigure}
    \caption{Left: the renderings of the object with three different orientations. Right: The object is rotated in three random axes from 0 to $2\pi$. The corresponding 2D images are embedded into a 3D latent space with \ours{}. For better visualization, the 3D latent representations are projected to 2D by PCA. The colors of the latent representations correspond to the rotation angles in $[0,2\pi]$. The mapping to latent space $\phi$ is continuous with respect to the $\mathrm{SO}(3)$ rotations, and each latent trajectory of rotations around a particular axis roughly forms a circular manifold. }
    \label{fig:rot3d}
\end{figure}

We present another example of learning equivariant representations from images. We consider a rotating bookshelf from ModelNet40 \citep{wu2015modelnet} and transform it in 3D through $\mathrm{SO}(3)$ rotations. The object is then rendered as a $48 \times 48$ image, which is the setting from \cite{shakerinava2022structuring}. Figure \ref{fig:rot3d} left displays the renderings of the object in three different orientations. The $\mathrm{SO}(3)$ action is nonlinear in the input space of 2D images. We use \ours{} to learn a latent space with 3 dimensions where the group action becomes linear.

Figure \ref{fig:rot3d} right shows three latent trajectories. Each trajectory is obtained by rotating the object around a randomly selected axis in 3D space. The colors of the latent representations correspond to the rotation angles in $[0,2\pi]$. The smooth transition of colors suggests that the mapping to latent space $\phi$ is continuous with respect to the $\mathrm{SO}(3)$ rotations. Also, each trajectory roughly forms a circular manifold.

We note that the trajectories are not in a perfect circular shape. For example, we observe that the latent representations overlap in some intervals. Concretely, given a particular rotation axis, let $x(\theta)$ denote the 2D rendering of the object with rotation angle $\theta$, and let $z(\theta) = \phi(x(\theta))$ denote its latent representation. In the 1st and 2nd latent trajectory shown in figure \ref{fig:rot3d}, it is observed that $z(\theta - \delta) \approx z(\theta + \delta)$ for some specific $\theta$ and small $\delta$'s. Also, in the 3rd trajectory, we have $z(\theta_1) \approx z(\theta_2)$ for some largly different $\theta_1$ and $\theta_2$. This can be caused by additional discrete symmetries in the object, where a transformation such as reflection or rotation up to $\pi$ leaves our view of the object unchanged. As our method is not provided with the group element associated with each object pose, it is unable to distinguish these identical inputs, so that they are mapped to the same location in the latent space and violate the overall circular structure. However, this kind of phenomenon does lead to an interesting question for future work: whether or not \ours{} can be extended to also discover these additional symmetries that are not caused by external transformations but lie in a real-world symmetric object itself.

\subsection{Learning the Latent Toroidal Structure of Flatland}\label{sec:flatland}

\begin{figure}[h]
    \centering
    \begin{subfigure}{.22\textwidth}
        \includegraphics[width=\textwidth]{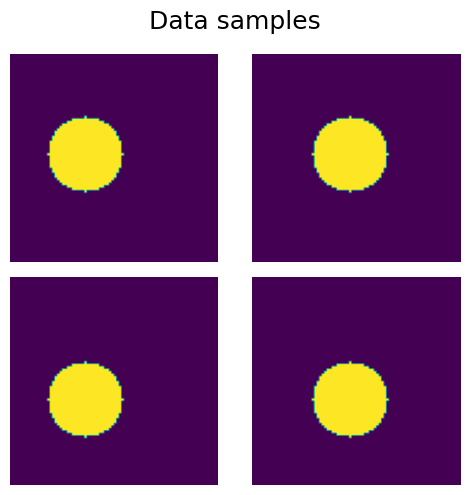}
    \end{subfigure}
    \begin{subfigure}{.75\textwidth}
        \includegraphics[width=\textwidth]{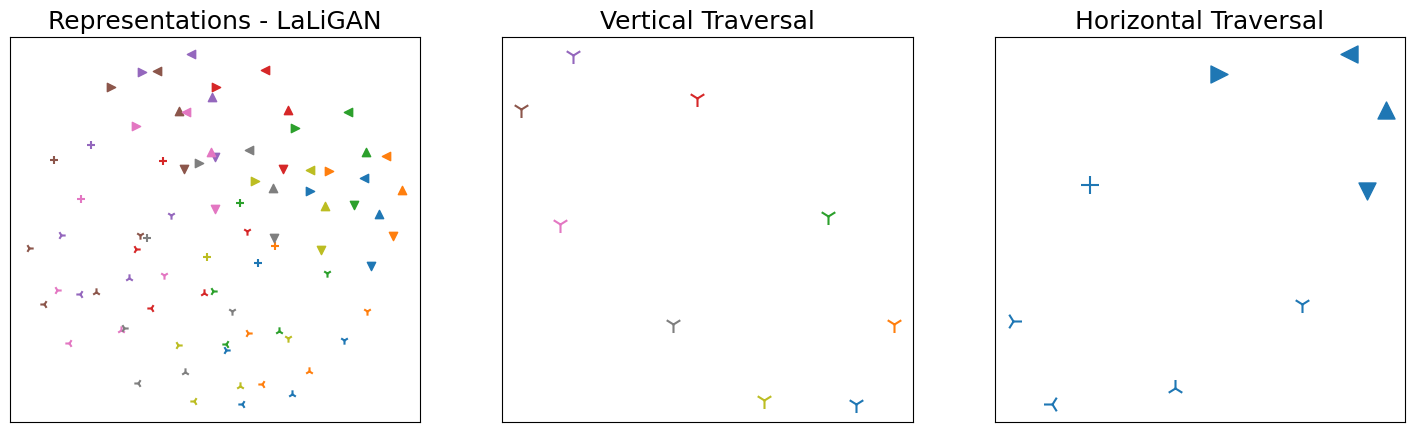}
    \end{subfigure}
    \caption{Left column: individual samples from the Flatland environment. Other columns: 2D projections of the 4D latent representations of equally spaced observations learned by \ours{}. Each marker style corresponds to a specific horizontal position of the ball. Each marker color corresponds to a specific vertical position. The latent space displays a toroidal structure, similar to the result in Figure 2 from \cite{quessard2020learning}.}
    \label{fig:flatland}
\end{figure}

Our method is related to symmetry-based disentangled representation learning \citep{higgins2018disentangled, caselles2019symmetry, quessard2020learning} in that both involves learning the group action. Though, our goal of symmetry discovery is intrinsically different from theirs. By symmetry, we refer to the equivariance of a function $f$, i.e. $f(gx)=gf(x)$. For example, $f$ can be the evolution function in a dynamical system: $x_{t+1}=f(x_t)$. But it would be more helpful to think of $f$ as an arbitrary function, e.g. an image classifier or a time series forecaster, written as $y=f(x)$. Then, our dataset $\{(x,y)\}$ consists of input-output pairs of this function. We discover the equivariance of the function from the input-output pairs. On the other hand, \citet{caselles2019symmetry} and \citet{quessard2020learning} use the group elements to describe the observational state transitions from $o_t$ to $o_{t+1}$. Their datasets are trajectories of $\{ o_0, g_0, o_1, g_1, ... \}$. They learn a map $f$ from observation $o \in W$ to latent $z \in Z$ that is equivariant between group actions on $W$ and $Z$.

For comparison, we consider a simple environment that is studied in these works, consisting of $84\times 84$ pixel observations of a ball moving in a plane. The world is cyclic, meaning that the ball will appear at the top if it crosses the bottom boundary, and similarly for left/right. The observations are shown in Figure \ref{fig:flatland} (left).

Unlike \citet{quessard2020learning} which considers sequences of observations $o$ and transformations $g$, $(o_0, g_0, o_1, g_1, ... )$, our goal is to discover the \textit{equivariance of a function}. Thus, we consider a function $o'=f(o)$ that simply translates the ball to the right and to the bottom by 15 pixels, respectively. An intuitive symmetry here is the cyclic translation equivariance along the two planar dimensions: if the input is translated by $g \in \mathrm{SO}(2) \times \mathrm{SO}(2)$, then the output will be translated by the same group element, i.e. $go' = f(go)$. In this experiment, we test whether \ours{} can discover a symmetry group of this function and a latent space where the group action becomes linear.

Following the setting in \citet{quessard2020learning}, we set the latent dimension to 4 and the search space of symmetries to $\mathrm{SO}(4)$. The discovered Lie algebra basis is
\begin{equation}
    L_1 = \begin{bmatrix}
        0 & 1.50 & -2.24 & 0 \\
        -1.50 & 0 & 0 & 0 \\
        2.24 & 0 & 0 & 0 \\
        0 & 0 & 0 & 0
    \end{bmatrix},\ 
    L_2 = \begin{bmatrix}
        0 & 0 & 0 & 0 \\
        0 & 0 & 0 & -4.25 \\
        0 & 0 & 0 & -2.86 \\
        0 & 4.25 & 2.86 & 0
    \end{bmatrix}
\end{equation}

It can be verified that this basis forms a valid Lie algebra that is closed under the Lie bracket. It is a commutative Lie algebra that matches the structure of $\mathrm{SO}(2) \times \mathrm{SO(2)}$. Note that we do not enforce any disentanglement in the learning process, so the latent dimensions are entangled. Disentanglement can be explicitly enforced by decomposing the latent space as independent subspaces as described in Section \ref{sec:latent-reg}, or promoted by encouraging the sparsity of the Lie algebra generators.

The toroidal structure of the latent space can be verified from Figure \ref{fig:flatland}. The visualization is obtained by projecting the 4D latent representations of equally spaced observations to 2D using Gaussian random projection. The marker colors and styles correspond to specific vertical and horizontal positions of the ball, respectively. It can be observed that all markers of a specific style, as well as all markers in a specific color, form a circular structure. For visual clarity, we also include two subsets: a vertical traversal along one column and a horizontal traversal along one row. This matches the result in Figure 2 from \citet{quessard2020learning}.

\subsection{Quantitative Evaluation of the Learned Symmetries}\label{sec:symm-eval}

\begin{table}[h]
    \centering
    \begin{tabular}{c|c|c|c}
    \hline
        Task & Symmetry & Equivariance error & Logit invariance error \\
    \hline
        \multirow{3}{*}{Reaction-Diffusion} & \ours{} & \textbf{1.02e-4} & \textbf{2.79e-3} \\
        & LieGAN & - & 3.11e-2 \\
        & $\mathrm{SO}(2)$ standard & 1.04e-4 & 2.84e-3 \\
    \hline
        \multirow{3}{*}{Lotka-Volterra} & \ours{} & \textbf{3.00e-2} & \textbf{5.21e-3} \\
        & LieGAN & 8.44e-2 & 4.05e-1 \\
        & $\mathrm{SO}(2)$ standard & 3.35e-2 & 5.68e-3 \\
    \hline
        \multirow{3}{*}{Pendulum} & \ours{} & \textbf{4.01e-3} & \textbf{5.33e-3} \\
        & LieGAN & 6.30e-3 & 2.11e-2 \\
        & $\mathrm{SO}(2)$ standard & 7.22e-3 & 1.57e-2 \\
    \hline
    \end{tabular}
    \caption{Quantitative metrics for the learned symmetries on test datasets.}
    \label{tab:quantitative-full}
\end{table}

In this section, we introduce some metrics to evaluate the discovered symmetries quantitatively. Recall that the symmetries are the equivariances of a function, i.e. $f(gx) = gf(x)$. Thus, a straightforward metric would be the equivariance error defined as

\begin{equation}
    EE_f = \mathbb E_{x,g} \|f(gx) - gf(x)\| ^2.
\end{equation}

Our symmetry discovery method is unsupervised and does not require fitting a function $f$. However, such a function can be fitted after discovering the symmetries, as is done in the dynamical system experiments. Concretely, the dataset consists of trajectories $\{x_{1:T}\}$, and the prediction function is $x_{t+1} = f(x_t)$. We use SINDy to learn symbolic equations $\dot z = h(z)$ (and therefore $z_{t+1} = H(z_t) = z_t + h(z_t)\Delta t$) in the latent space as shown in Table \ref{tab:eq-rd-3d}, \ref{tab:eq-pd} and \ref{tab:eq-lv}. Then, $x_{t+1}=f(x_t)=(\psi\circ H \circ\phi)(x_t)$, where $\phi$ and $\psi$ are the learned encoder and decoder. Using this function $f$, we can evaluate the equivariance error of the learned symmetries on the test datasets. For comparison, we include the symmetry learned by linear LieGAN \citep{liegan} in the input space (without autoencoder), where the function $f$ is the SINDy model trained in the input space (third row (SINDy) of Table \ref{tab:eq-pd} and \ref{tab:eq-lv}). Note that this result is unavailable for the high-dimensional reaction-diffusion system because we did not train SINDy on its input space. Besides, we use the same autoencoder but replace the representation learned \ours{} with a standard representation of $\mathrm{SO}(2)$, i.e. $L = [0,-1; 1,0] \in \mathbb R^{2 \times 2}$. Table \ref{tab:quantitative-full} shows that \ours{} reaches the lowest equivariance errors on all of the three dynamical systems.

Another quantitative metric is inspired by the logit invariance introduced in \citet{moskalev2023genuine}. For a classification task, we define the logit invariance error to measure the change of logits under group actions:
\begin{equation}
    LI_f = \mathbb E_{x,g} \frac{1}{2} \| f(x) - f(gx) \| ^2.
\end{equation}

Here, the function $f$ outputs the logits for classification. In our setting, there is not necessarily such a classification function. However, we can utilize the learned discriminator in \ours{}, which effectively classifies between the original data distribution and the transformed distribution by the symmetry generator. A good symmetry should lead to a small difference between these two distributions. Therefore, we define the discriminator logit invariance error as follows:
\begin{equation}
    DLI = \mathbb E_{v,g} \frac{1}{2} \| D(v) - D(gv) \| ^2
\end{equation}
where $v=(x,y)$ are the data points sampled from the dataset. Table \ref{tab:quantitative-full} shows that \ours{} has the lowest discriminator logit invariance error among the considered symmetries.

%% file: appendix/prelim.tex
\section{Preliminaries on Lie Group Representations}\label{sec:lie-prelim}

A Lie group is both a group and a differentiable manifold. We use Lie groups to describe continuous symmetry transformations. For example, the rotations in $\mathbb R^n$ form a Lie group $\mathrm{SO}(n)$; all rotations, translations and reflections in $\mathbb R^n$ form the Euclidean group $\mathrm{E}(n)$. We also referred to general linear group $\mathrm{GL}(n;\mathbb R)$, which is the group of all $n \times n$ invertible matrices with real entries. As we only consider the field of real numbers in this work, we sometimes omit $\mathbb R$ and write $\mathrm{GL}(n)$ instead. We may also write $\mathrm{GL}(V)$, which is equivalent to $\mathrm{GL}(n; \mathbb R)$ if $V=\mathbb R^n$ is a vector space.

The tangent vector space at the identity group element is called the Lie algebra of the Lie group $G$, denoted as $\mathfrak g = T_\mathrm{id} G$. The Lie algebra of general linear group $\mathrm{GL}(n, \mathbb R)$ consists of all real-valued matrices of size $n \times n$. As Lie algebra is a vector space, we can use a basis $L_i \in \mathfrak g$ to describe any of its element as $A = \sum_{i=1}^cw_iL_i$, where $w_i \in \mathbb R$ and $c$ is the dimension of the vector space. Lie algebra can be interpreted as the space of infinitesimal transformations of the group. Group elements infinitesimally close to the identity can be written as $g = I + \sum_{i=1}^cw_iL_i$.

The exponential map $\exp: \mathfrak g \rightarrow G$ gives a mapping from the Lie algebra to the Lie group. For matrix Lie groups that we are considering, matrix exponential is such a map.

We are interested in how the data is transformed by group elements. Lie group, just like any other group, transforms the data from a vector space via a group action $\alpha: G\times V \rightarrow V$. If the action is linear, we call it a Lie group representation $\rho: G \rightarrow \mathrm{GL}(V)$, which acts on the vector space $V$ by matrix multiplication. Such a group representation induces a Lie algebra representation, $d\rho: \mathfrak g \rightarrow \mathfrak{gl}(V)$, which satisfies $\exp (d\rho (L)) = \rho (\exp(L)),\forall L \in \mathfrak g$.

Every matrix Lie group $G \leq \mathrm{GL}(n)$ has a standard representation, which is just the inclusion map of $G$ into $\mathrm{GL}(n)$. In our work, as we only consider these subgroups of general linear group, we learn the Lie group as its standard representation acting on $\mathbb R^n$ in the usual way. It is thus convenient to think of all group elements (and also Lie algebra elements) as $n \times n$ matrices, with the group operation given by matrix multiplication.

%% file: appendix/nga-theory.tex
\section{Universal Approximation of Nonlinear Group Actions}
\subsection{Proofs}\label{sec:uanga}
In this section, we provide theoretical justifications for the decomposition of nonlinear group actions introduced in section \ref{sec:decomp}. We represent any nonlinear group action $\pi': G \times \mc{M} \rightarrow \mc{M}$ as
\begin{equation}
\pi'(g, \cdot) = \psi \circ \pi(g) \circ \phi \big|_\mc{M},
\label{eq:decomp_2}
\end{equation}
where $\phi: V \rightarrow Z$ and $\psi: Z \rightarrow V$ are functions parametrized by neural networks, and $\pi(g): G \rightarrow \mathrm{GL}(k)$ is a group representation acting on the latent vector space $Z=\mathbb R^k$.

\begin{proposition}
    $\pi'(g,\cdot) = \psi \circ \pi(g) \circ \phi \big|_\mc{M}$ is a group action on $\mc{M}$ if (1) $\phi \big|_\mc{M}$ is the right-inverse of $\psi$, and (2) the image of $\mc M$ under $\phi$ is invariant under the action $\pi$ of $G$, i.e. $G\phi[\mc M] = \phi[\mc M]$.
\end{proposition}

\begin{proof}
We prove that $\pi'$ defined this way indeed satisfies the identity and the compatibility axioms of group action. The identity condition is obvious from the property of right-inverse:
\begin{equation}
    \pi'(e,x) = \psi(\phi(x)) = x
\end{equation}

As $G\phi[\mc M] = \phi[\mc M]$, for any $x \in \mc M$ and $g \in G$, $\exists x' \in \mc M$ s.t. $\pi(g)\phi(x) = \phi(x')$. Then, we can conclude that $\psi$ is injective when restricted to $G \phi[\mc M]$ from the right inverse property: $\psi(\phi(x_1)) = \psi(\phi(x_2)) \Rightarrow x_1 = x_2 \Rightarrow \phi(x_1) = \phi(x_2)$.

Then, denoting $z = \phi(x)$ and $gz = \pi(g)z$, we have $\psi(gz) = \psi(gz) 
 \Rightarrow \psi(\phi(\psi(gz))) = \psi(gz) \Rightarrow \phi(\psi(gz)) = gz$ for any $x \in \mc M$ and $g \in G$, and thus
\begin{align}
    \pi'(g_2, \pi'(g_1, x))&= \psi(\pi(g_2) \phi(g_1\cdot x))\cr 
    &=\psi(\pi(g_2) \phi(\psi(\pi(g_1) \phi(x)))) \cr 
    &=\psi(\pi(g_2) \pi(g_1) \phi(x)) \cr 
    &=\psi(\pi(g_2 g_1) \phi(x)) \cr 
    &=\pi'(g_2 g_1, x) 
\end{align}
\end{proof}

The following theorem states that our proposed decomposition and neural network parametrization can approximate nonlinear group actions under certain conditions.

\begin{theorem}[Universal Approximation of Nonlinear Group Action]\label{thm:uanga_2}
    Let $G\leq\mathrm{GL}(k; \mathbb R)$ be a compact Lie group that acts smoothly, freely and properly via a continuous group action $\pi':G\times \mc{M}\rightarrow \mc{M}$, where the data manifold $\mc{M}$ is a compact subset of $V=\mathbb R^n$. The group action, restricted to any bounded subset of the group, can be approximated by the decomposition $\pi'(g,\cdot)\approx\psi\circ\pi(g)\circ\phi$ if it admits a simply connected orbit space $\mc{M}/G$, where $\psi$ and $\phi$ are fixed arbitrary-width neural networks with one hidden layer, and $\pi$ is a linear group representation.
\end{theorem}

\begin{proof}
    We establish our theorem as a corollary of the Universal Approximation Theorem (UAT) \citep{uat}, which states that any continuous function $f\in \mathcal C (\mathbb R^n, \mathbb R^m)$ can be approximated by a one-hidden-layer arbitrary-width neural network. The intuition of this proof is to explicitly construct the mappings between input and latent space and ensure their continuity so that we can use UAT to approximate them with neural nets.

    The Quotient Manifold Theorem states that smooth, free, and proper group actions yield smooth manifolds as orbit spaces (\citet{lee2012smooth}, Theorem 21.10). More precisely, the orbit space $\mc{M}/G$ has a unique smooth structure with a smooth submersion quotient map $s: \mc{M} \rightarrow \mc{M}/G$. Also, given that $\mc{M}/G$ is simply connected, we can find a global continuous section $s': \mc{M}/G \rightarrow \mc{M}$ s.t. $s' \circ s$ is identity restricted on $S = \mathrm{img}_{s'} (\mc{M}/G)$. The global section can be constructed by:
    \begin{enumerate}
        \item Fix a base point $p \in \mc{M}/G$ and choose a point $\tilde p$ in the pre-image of $p$ under $s$, i.e. $s(\tilde p) = p$.
        \item For any other point $q \in \mc{M}/G$, choose a path $\gamma$ in $\mc{M}/G$ from $p$ to $q$.
        \item As $\mc{M}/G$ is simply connected, $\mc{M}$ is a universal cover of $\mc{M}/G$, so that any path $\gamma$ in $\mc{M}/G$ can be uniquely lifted to a path $\tilde\gamma$ in $\mc{M}$ which starts at $\tilde p$ and ends at $\tilde q$.
        \item Define the section as $s': \mc{M}/G \rightarrow \mc{M}, q \mapsto \tilde q$.
    \end{enumerate}

    In addition, according to Whitney Embedding Theorem, the smooth manifold $\mc{M}/G$ can be smoothly embedded in a higher-dimensional Euclidean space. Denote $t: \mc{M}/G \rightarrow \mathbb R^p$ as one of the possible embeddings. We do not restrict the exact dimensionality of such an Euclidean space, as long as it enables us to represent any orbit with part of the latent space.
    

    Before defining the mapping from input to latent, we finally note that
    as $G\leq\mathrm{GL}(k; \mathbb R)$, we have a standard representation $\rho: G\rightarrow\mathbb R^{k\times k}$.

    Now we define $\alpha:\mc{M} \rightarrow \mathbb R^{k^2+p}, \pi'(g, s'(\tilde v)) \mapsto \text{concat} (\mathrm{vec}(\rho(g)), t(\tilde v)), \forall \tilde v \in \mc{M}/G, g\in G$, and we verify that this function is well defined.
    
    First, $\mc{M}=\{\pi'(g, s'(\tilde v)) | \tilde v \in \mc{M}/G, g \in G\}$, so that $\alpha(v)$ is defined for any $v\in \mc{M}$.
    
    Then, we need to make sure any $v\in \mc{M}$ is written uniquely in the form of $v=\pi'(g) \tilde v$. $\forall \tilde {x_i}\not= \tilde{x_j}, g_i, g_j \in G, \pi'(g_i) \tilde {x_i} \not= \pi'(g_j) \tilde {x_j}$, because any two orbits never overlap in $\mc{M}$.
    
    Also, $\forall g_1,g_2\in G, g_1\not= g_2$, as $\pi'$ acts freely, we have $\pi'(g_1) \tilde v\not=\pi'(g_2) \tilde v$.

    Next, we prove that $\alpha$ defined this way is also continuous. As the value of $\alpha$ is concatenated from two parts, it suffices to check the continuity for each component, i.e. $\alpha_1(\pi'(g, s'(\tilde v))) = \mathrm{vec}(\rho(g))$ and $\alpha_2(\pi'(g, s'(\tilde v))) = t(\tilde v)$.

    For any open set $t(\tilde V) \subset \mathbb R^p$, where $\tilde V \subset \mc{M} / G$, the continuity of $t$ and $s$ guarantees that the inverse image, $(t \circ s)^{-1} t(\tilde V) = s^{-1} (\tilde V)$, is an open set. As $(s|_S)^{-1}=s'$, $s'(\tilde V)$ is an open set. The $\alpha_2$ inverse image of $t(\tilde V)$ is $\bigcup_{g \in G}\pi'(g, s'(\tilde V))$. Note that $\forall g \in G, \pi'(g^{-1}, \cdot): \mc{M} \rightarrow \mc{M}$ is continuous, so that $\pi'(g, s'(\tilde V))$ is open. Therefore, the $\alpha_2$ inverse image of any open set $t(\tilde V)$ is a union of open sets, which is also open, so that $\alpha_2$ is continuous.

    Similarly, for any open set $\mathrm{vec}(\rho(U)) \in \mathbb R^k$, $U$ is an open set given the continuity of the standard representation $\rho$ and the vectorization operation. The $\alpha_1$ inverse image of $\mathrm{vec}(\rho(U))$ is $\bigcup_{s'(\tilde v)} \pi'(U, s'(\tilde v))$. As the action of $G$ on $\mc{M}$ is free, i.e. the stabilizer subgroup is trivial for all $v \in \mc{M}$, we have $\pi'(\cdot, v): G \rightarrow \mc{M}$ is an injective continuous map, so that its image of an open set is still open. Thus, we conclude that $\alpha_1$ is also continuous.
    
    Given that the data manifold $\mc{M}$ is a closed subset of the ambient space $V=\mathbb R^n$, the Tietze extension theorem \cite{dugundji1951extension} ensures that $\alpha: \mc{M} \rightarrow \mathbb R^{k^2+p}$ can be continuously extended to a function on $V$. According to the Universal Approximation Theorem, there exists a one-hidden-layer arbitrary-width neural network $\phi$ that approximates the continuous extension of $\alpha$.

    Then, we define $\pi(g) = (I_k \otimes \rho(g)) \oplus I_p$. For some $z_0 = (\mathrm{vec}(\rho(g_0)), t(\tilde v_0))$ in the image of $\alpha$, we have
    \begin{align*}
        \pi(g)z_0 =& ((I_k \otimes \rho(g))\mathrm{vec}(\rho(g_0)), t(\tilde v_0)) \\
        =& (\mathrm{vec}(\rho(g)\rho(g_0)), t(\tilde v_0)) \\
        =& (\mathrm{vec}(\rho(gg_0)), t(\tilde v_0))
    \end{align*}
    
    Finally, we define another mapping $\beta$ on $GZ = \bigcup_{g \in G, z \in Z}\pi(g)z$, where $Z$ is the image of $\alpha$, as $\beta: (\mathrm{vec}(\rho(g)), t(\tilde v)) \mapsto \pi'(g, s'(\tilde v)), \forall \tilde v \in \mc{M}/G, g\in G$. It is well-defined because $\mathrm{vec} \circ \rho$ is injective on $G$, and also continuous because it is the inverse of $\alpha$. Similarly, we need to extend the function on $GZ$ to the entire vector space. Because $\mc M$ is a compact set, its image $Z$ under the continuous function $\alpha$ is also compact (and therefore closed) in $\R^{k^2+p}$. Then, the proper action of the compact group $G$ ensures that the image of the group action, i.e. $GZ$ is also compact. Thus, we can continuously extend $\beta$ from $GZ$ to $\R^{k^2+p}$. According to the Universal Approximation Theorem, there exists another neural network $\psi$ that approximates $\beta$.

    Finally, defining $\alpha,\pi,\beta$ as above, for any $v=\pi'(g', s'(\tilde v)) \in \mc{M}$ and $g$ in any bounded subset of $G$, we have
    \begin{align*}
        \pi'(g,v)=&\pi'(g, \pi'(g', s'(\tilde v)))\\
                  =&\pi'(gg', s'(\tilde v))\\
                  =&\beta(\mathrm{vec}(\rho(gg')), t(\tilde v))\\
                  =&(\beta\circ\pi(g))(\mathrm{vec}(\rho(g')), t(\tilde v))\\
                  =&(\beta\circ\pi(g)\circ\alpha)(\pi'(g', s'(\tilde v))\\
                  =&(\beta\circ\pi(g)\circ\alpha)(v) \\
                  \approx&(\psi\circ\pi(g)\circ\phi)(v)
    \end{align*}

    The final step relies on the fact that the neural network approximator $\psi$ and the group representation $\pi(g)$ are Lipschitz continuous. Concretely, it requires $\|\psi(z_1)-\psi(z_2)\| \leq K \| z_1 - z_2 \|,\ \forall z_1, z_2,$ for some positive constant $K$ and similarly for $\pi(g)$ as a function over $Z=\mathrm{img}(\alpha)$. This is true for a one-layer neural network with ReLU activation, and also for $\pi(g)$ for any $g$ in a bounded subset of the group, because $\pi(g)$ is a bounded linear transformation.

    Then, according to the UAT, for any $\epsilon > 0$, there exist neural networks $\psi$ and $\phi$ and positive constant $K$ s.t.
    \begin{align*}
        &\sup_{v \in \mc{M}} \|(\psi\circ\pi(g)\circ\phi)(v) - (\beta\circ\pi(g)\circ\alpha)(v)\| \\
        \leq & \sup_{v \in \mc{M}} \|(\psi\circ\pi(g)\circ\phi)(v) - (\psi\circ\pi(g)\circ\alpha)(v)\| + \|(\psi\circ\pi(g)\circ\alpha)(v) - (\beta\circ\pi(g)\circ\alpha)(v)\| \\
        \leq & \sup_{v \in \mc{M}} K\|(\pi(g)\circ\phi)(v) - (\pi(g)\circ\alpha)(v)\| + \epsilon \\
        \leq & \sup_{v \in \mc{M}} K^2 \| \phi(v) - \alpha(v) \| + \epsilon \\
        \leq & (K^2 + 1) \epsilon
    \end{align*}
    which translates to
    \begin{align*}
        (\beta\circ\pi(g)\circ\alpha)(v) \approx (\psi\circ\pi(g)\circ\phi)(v)
    \end{align*}
    
    
\end{proof}

\subsection{Group Action Under Approximate Inverse}\label{sec:g-action-approx}

In practice, the networks $\phi$ and $\psi$ are trained with a reconstruction loss. As the loss is not strictly zero, they are only approximate but not perfect inverses of each other. As a result, the condition in Proposition \ref{prop:nga} cannot be strictly true. However, we can show empirically that when the reconstruction loss is reasonably close to zero, the decomposition in Proposition \ref{prop:nga} leads to an approximate group action. We use the reaction-diffusion system for demonstration.

A group action needs to satisfy the identity and compatibility axioms. We evaluate the error in terms of these axioms caused by the imperfect encoder and decoder networks. First, the error with respect to the identity axiom can be directly described by the reconstruction loss:
\begin{equation}
    \mathrm{err_{id}} = \mathbb E_x \| \pi'(e,x) - x \| ^2 = \mathbb E_x \| \psi(\phi(x)) - x \| ^2 = l_\text{recon}.
\end{equation}
In the reaction-diffusion experiment, the test reconstruction loss is $2.58 \times 10^{-3}$, which indicates the autoencoder networks approximately satisfy the identity axiom.

Then, we consider the compatibility error. We sample a random group element $g$ from the generator and calculate $g^N$. Then, we apply $\pi'(g)^N = (\psi \circ \pi(g) \circ \phi)^N$ and $\pi'(g^N) = \psi \circ \pi(g^N) \circ \phi$ to the test dataset. The compatibility error is computed as
\begin{equation}
    \mathrm{err_{comp}} = \mathbb E_x \| \pi'(g)^N(x) - \pi'(g^N)(x) \|^2.
\end{equation}

\begin{figure}[ht]
    \centering
    \begin{subfigure}{.68\textwidth}
        \includegraphics[width=\textwidth]{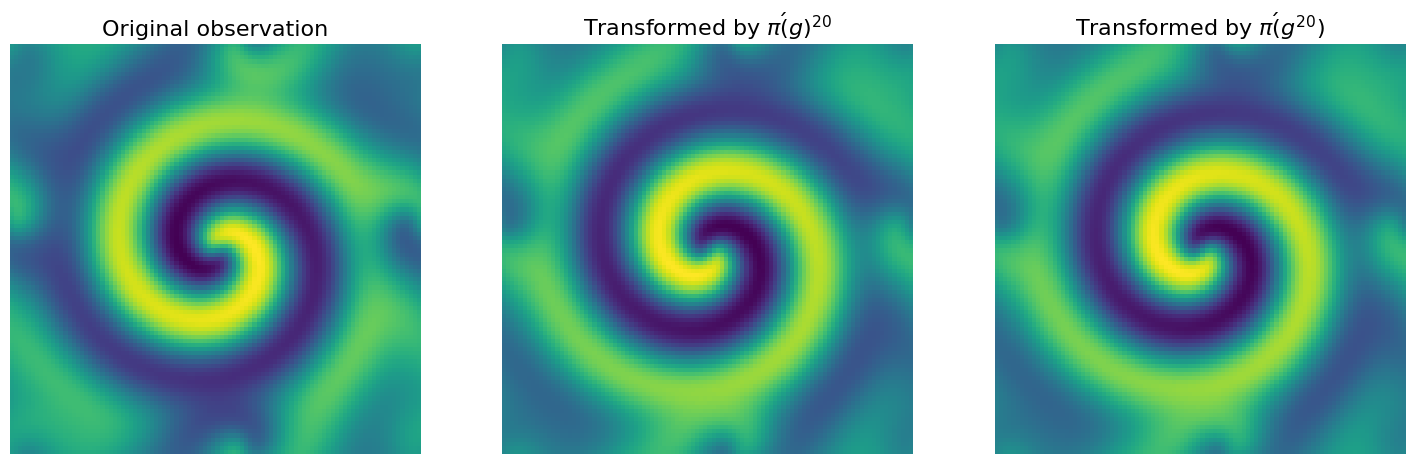}
        \caption{An observation (left) transformed by $\pi'(g)^{20}$ (middle) and $\pi'(g^{20})$ (right)}
        \label{fig:nga-vrf-vis}
    \end{subfigure}
    \begin{subfigure}{.3\textwidth}
        \includegraphics[width=\textwidth]{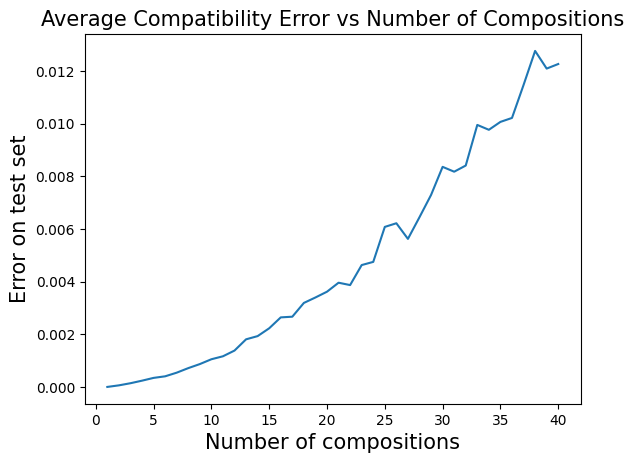}
        \caption{Compatibility error vs $N$}
        \label{fig:nga-vrf-curve}
    \end{subfigure}
    \caption{Group compatibility error caused by imperfect encoder and decoder networks.}
    \label{fig:nga-vrf}
\end{figure}

Figure \ref{fig:nga-vrf-vis} shows a sample from test set transformed by $\pi'(g)^{20}$, that is, passed through the encoder, the linear representation $\pi(g)$ and the decoder for 20 times, and by $\pi'(g^{20})$, that is, passed through the autoencoder and the linear representation $\pi(g^{20})$ once. The two transformations have the same effect visually, which indicates that the autoencoder networks approximately satisfy the compatibility axiom. Further evidence is provided in Figure \ref{fig:nga-vrf-curve}, where we use the number of compositions $N \in [2, 40]$ and plot the growth of compatibility error with the increase of $N$. The error remains low ($\approx 1 \times 10^{-2}$) up to 40 times of group element composition.

\subsection{Notes on Latent Regularizations}
In \cref{sec:latent-structure}, we introduced two strategies to regularize the latent space for easier discovery of symmetries. We note that our model is still able to learn all the desired symmetries after applying these regularizations.

First, the orthogonal parametrization would not remove symmetry. If we have an encoder with non-orthogonal final layer $W$, we can apply the Gram-Schmidt process to get an orthogonal weight $Q = PW$, which is effectively a change of basis in the latent space. If the original encoder weight $W$ leads to a latent space with linear symmetry, we can reconstruct the symmetry with a different group representation based on the change of basis. Thus, an orthogonal final layer suffices to learn all desired symmetries.

\begin{wrapfigure}{r}{6cm}
    \centering
    \includegraphics[width=.3\textwidth]{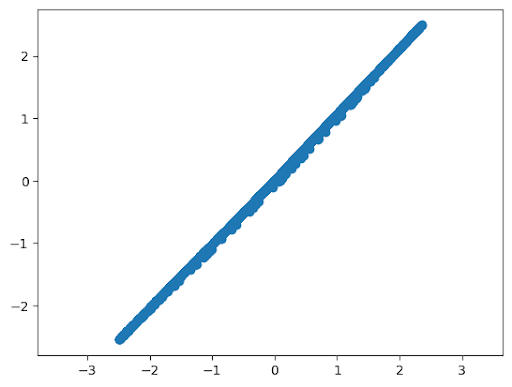}
    \caption{Latent embeddings of the observations $x \in \R^{100 \times 100}$ of the reaction-diffusion system without orthogonal parametrization in the encoder.}
    \label{fig:rd-fail-1dlatent}
\end{wrapfigure}

Also, the zero-mean normalization does not affect the symmetries of general linear groups. As is stated in \cref{sec:latent-structure}, the linear group representation $v \mapsto \pi(g)v$ implies that the vector space should be centered at the origin. Otherwise, the group transformations will transfer the data distribution to a new center, and the resulting distribution cannot be the same as the original distribution, leading to failure in symmetry discovery.

We also provide an additional example here to demonstrate how the orthogonal parametrization can be helpful in practice. We consider the reaction-diffusion system in \cref{sec:exp-dynamics} (\cref{fig:rd-2d}). We fix all other settings and only change the orthogonal final layer in the encoder to a regular linear layer. The discovered latent space is shown in \cref{fig:rd-fail-1dlatent}. Instead of the circular shape in \cref{fig:rd-latent-2d}, the latent representations collapse into a 1D line segment $z_1=z_0$, oscillating between two endpoints. The discovered symmetry generator has roughly the form $L=\begin{bmatrix}
    a & -a \\ -b & b
\end{bmatrix}$, which acts (approximately) as identity restricted to this 1D subspace.

%% file: appendix/exp_detail.tex
\section{Experiment Details}
In this section, we provide the detailed hyperparameter settings and dataset generation procedure for the experiments.

\subsection{Reaction-Diffusion}
We use the script from SINDy Autoencoder \footnote{\href{https://github.com/kpchamp/SindyAutoencoders/tree/master/rd\_solver}{https://github.com/kpchamp/SindyAutoencoders/tree/master/rd\_solver}} to generate the dataset. We discretize the 2D space into a $100\times100$ grid, which leads to an input dimension of $10^4$. We simulate the system up to $T=6000$ timesteps with step size $\Delta t=0.05$. Then, we add random Gaussian noises with standard deviation $10^{-6}$ to each pixel and at each timestep. We use the timesteps $t \in [0, 4800)$ for training \ours{} and SINDy. For long-term forecasting, we use the timestep $t = 4800$ as the initial input frame and simulate up to $600$ timesteps with each method. The simulations are then compared with the ground truth during $t \in [4800, 5400)$ to calculate the relative MSE.

We use MLPs with 5 hidden layers and 512 hidden units as the encoder, the decoder and the discriminator. We also use orthogonal parametrization for the final linear layer of the encoder, which is discussed in Section \ref{sec:latent-reg}. The dimension of the Lie algebra in the LieGAN generator is set to one. We use a standard Gaussian as the distribution of the coefficient $w$ in the LieGAN generator.

For the 2D latent space symmetry discovery, we train for 150 epochs with batch size 64. The learning rates for the autoencoder, the generator and the discriminator are $0.0003, 0.001, 0.001$, respectively. The weights of the reconstruction loss and the GAN loss are set to $w_\text{recon}=1$ and $w_\text{GAN}=0.01$. As in LieGAN, we also include a regularization loss term $l_\text{reg}$ for LieGAN generator, which pushes the Lie algebra basis away from zero, and the weight for the regularization is set to $w_\text{reg}=0.1$. We also apply sequential thresholding to the LieGAN generator parameters. Every 5 epochs, matrix entries with absolute values less than $0.01$ times the max absolute values across all entries are set to 0. For the 3D latent space, the settings are the same as above except that we train for 300 epochs.


\subsection{Nonlinear Pendulum}
We simulate the movement of nonlinear pendulum according to the governing equation, $\dot q = p, \dot p = =-\sin (q)$. For training, we simulate $200$ trajectories up to $T=500$ timesteps with $\Delta t=0.02$ with random initial conditions. For testing, we simulate another $20$ trajectories. The initial conditions are sampled uniformly from $q_0 \in [-\pi, \pi]$ and $p_0 \in [-2.1, 2.1]$. Also, we ensure that $\mathcal H=\frac{1}{2}p^2 - \cos (q) < 0.99$, so that it does not lead to a circular movement.

We use MLPs with 5 hidden layers and 512 hidden units as the encoder, the decoder and the discriminator. We also use orthogonal parametrization for the final linear layer of the encoder and batch normalization before the transformation of the symmetry generator, as discussed in Section \ref{sec:latent-reg}. The dimension of the Lie algebra in the LieGAN generator is set to one. We use a standard Gaussian as the distribution of the coefficient $w$ in the LieGAN generator.

We train for 70 epochs with batch size 256. The learning rate for the autoencoder, the generator and the discriminator are all $0.001$. The weights of the reconstruction loss and the GAN loss are set to $w_\text{recon}=1$ and $w_\text{GAN}=0.01$. The weight for the LieGAN regularization is set to $w_\text{reg}=0.02$. We also apply sequential thresholding to the LieGAN generator parameters. Every 5 epochs, matrix entries with absolute values less than $0.3$ times the max absolute values across all entries are set to 0.


\subsection{Lotka-Volterra Equations}
We simulate the Lotka-Volterra equations in its canonical form, $\dot p = a - b e ^ q, \dot q = c e ^ p - d$, with $a = 2/3, b = 4/3, c=d=1$. For training, we simulate $200$ trajectories up to $T=10000$ timesteps with $\Delta t=0.002$ with random initial conditions. For testing, we simulate another $20$ trajectories. The initial conditions are sampled by first sampling $x_0 = e ^ {p_0}$ and $y = e ^ {q_0}$ uniformly from $[0, 1]$ and then computing $p_0 = \log x_0$ and $q_0 = \log y_0$. Also, we ensure that the Hamiltonian of the system given by $\mathcal H=c e ^ p - d p + b e ^ q - a q$ falls in the range of $[3, 4.5]$.

For all the experiments, we use MLPs with 5 hidden layers and 512 hidden units as the encoder, the decoder and the discriminator. We also use orthogonal parametrization for the final linear layer of the encoder and batch normalization before the transformation of the symmetry generator, as discussed in Section \ref{sec:latent-reg}. The dimension of the Lie algebra in the LieGAN generator is set to one. We use a standard Gaussian as the distribution of the coefficient $w$ in the LieGAN generator.

We train for 30 epochs with batch size 8192. The learning rate for the autoencoder, the generator and the discriminator are all $0.001$. The weights of the reconstruction loss and the GAN loss are set to $w_\text{recon}=1$ and $w_\text{GAN}=0.01$. The weight for the LieGAN regularization is set to $w_\text{reg}=0.01$. We also apply sequential thresholding to the LieGAN generator parameters. Every 5 epochs, matrix entries with absolute values less than $0.3$ times the max absolute values across all entries are set to 0.


\subsection{Double Bump}
The signal length is set to 64, so that we have observations $x \in \mathbb R^{64}$. The rectangular and the triangular bump signals both have the length 16. For each sample, we randomly sample a shift $(\Delta_1,\Delta_2)$, where $\Delta_i$ is an integer in $[0, 64)$. The two bump signals are then cyclically shifted and superimposed. We sample $10000$ signals for training and another $1000$ for test.

We use a 1D convolution architecture for autoencoder. The encoder consists of three 1D convolution layers, with the numbers of input channels 1, 16, 32 and the final number of output channels 64, kernel size 3, stride 1 and padding 1, each followed by ReLU activation and a 1D max pooling layer with kernel size 2 and stride 2. The output of the final convolution is flattened and fed into an MLP with 2 hidden layers with 128 and 32 hidden units, and 4 output dimensions. The decoder structure is the reverse of the encoder structure, It consists of a 2-layer MLP with 32 and 128 hidden units, and 512 output dimensions. The MLP output is reshaped into 64 channels with size 8. Then three transposed convolution layers with output channels 32, 16, 1, kernel size 3, stride 2, input padding 1 and output padding 1 are applied. The final output passes through a sigmoid activation to ensure the output range is in $(0,1)$. We use MLPs with 4 hidden layers and 128 hidden units as the discriminator. We also use orthogonal parametrization for the final linear layer of the encoder, as discussed in Section \ref{sec:latent-reg}. The Lie algebra basis in the LieGAN generator is fixed to the standard representation of $\mathrm{SO}(2) \times \mathrm{SO}(2)$.

We train for 2000 epochs with batch size 64. The learning rate for the autoencoder and the discriminator are both $0.001$. The weights of the reconstruction loss and the GAN loss are set to $w_\text{recon}=1$ and $w_\text{GAN}=0.01$.